%% file: main.tex
\documentclass[10pt,twocolumn,letterpaper]{article}

\usepackage{cvpr}              


\definecolor{cvprblue}{rgb}{0.21,0.49,0.74}
\usepackage[pagebackref,breaklinks,colorlinks,allcolors=cvprblue]{hyperref}

\usepackage[accsupp]{axessibility}  

\usepackage[utf8]{inputenc} 
\usepackage[T1]{fontenc}    
\usepackage{url}            
\usepackage{booktabs}       
\usepackage{amsfonts}       
\usepackage{nicefrac}       
\usepackage{microtype}      
\usepackage{xcolor}  
\usepackage{pifont}
\usepackage{nicefrac}
\usepackage{tcolorbox}
\usepackage{colortbl}

\usepackage[ruled,vlined]{algorithm2e}

\usepackage{makecell}
\usepackage{booktabs} 
\usepackage{multirow}
\usepackage{amssymb}
\usepackage{mathrsfs}
\usepackage{amsmath}
\usepackage{amsthm}
\usepackage{pifont}
\usepackage[capitalize]{cleveref}

\theoremstyle{plain}
\newtheorem{theorem}{Theorem}[section]
\newtheorem{proposition}[theorem]{Proposition}

\theoremstyle{definition}

\theoremstyle{remark}

\crefname{equation}{Eq.}{Eqs.}
\Crefname{equation}{Equation}{Equations}
\crefname{figure}{Fig.}{Figs.}
\Crefname{figure}{Figure}{Figures}
\crefname{table}{Tab.}{Tabs.}
\Crefname{table}{Table}{Tables}
\crefname{algocf}{Alg.}{Algs.}
\Crefname{algocf}{Algorithm}{Algorithms}
\crefname{section}{Sec.}{Secs.}
\Crefname{section}{Section}{Sections}
\crefname{appendix}{App.}{Apps.}
\Crefname{appendix}{Appendix}{Appendices}

\crefname{theorem}{Thm.}{Thms.}
\Crefname{theorem}{Theorem}{Theorems}
\crefname{lemma}{Lem.}{Lems.}
\Crefname{lemma}{Lemma}{Lemmas}
\crefname{definition}{Def.}{Defs.}
\Crefname{definition}{Definition}{Definitions}
\crefname{corollary}{Cor.}{Cors.}
\Crefname{corollary}{Corollary}{Corollaries}
\crefname{remark}{Rem.}{Rems.}
\Crefname{remark}{Remark}{Remarks}
\crefname{proposition}{Prop.}{Props.}
\Crefname{proposition}{Proposition}{Propositions}
\crefname{proof}{Pr.}{Prs.}
\Crefname{proof}{Proof}{Proofs}

\providecommand{\bsym}[1]{\boldsymbol{\mathrm{#1}}}

\providecommand{\rieexp}{\mathrm{Exp}}
\providecommand{\rielog}{\mathrm{Log}}

\providecommand{\bscal}[1]{\boldsymbol{\mathcal{#1}}}

\definecolor{lightgray}{gray}{0.9}

\def\paperID{3448} 
\def\confName{CVPR}
\def\confYear{2025}

\title{Learning to Normalize on the SPD Manifold under Bures-Wasserstein Geometry}

\author{Rui Wang$^{1}$\thanks{Equal contribution.}, Shaocheng Jin$^{1}$\footnotemark[1],  Ziheng Chen$^{2}$\thanks{Corresponding Author}, Xiaoqing Luo$^{1}$, Xiao-Jun Wu$^{1 }$\\
\textsuperscript{1}School of Artificial Intelligence and Computer Science, Jiangnan University, Wuxi, China\\
\textsuperscript{2}Department of Information Engineering and Computer Science, University of Trento, Trento, Italy\\
{\tt \{cs\_wr, xqluo, wu\_xiaojun\}@jiangnan.edu.cn}\\
{\tt \{6233112021\}@stu.jiangnan.edu.cn, ziheng\_ch@163.com}
}

\begin{document}
\maketitle
\input{Sections/Abstract}
\section{Introduction}
\label{sec:intro}
\input{Sections/Introduction}

\section{Preliminaries}
\label{sec:Preli}
\input{Sections/Preliminaries}
\section{Proposed method}
\label{sec:Propose}
\input{Sections/Proposed_method}

\input{Sections/Experiment}

\section{Conclusion}
\label{sec:conclusion}
\input{Sections/Conclusion}

\clearpage
\section*{Acknowledgments}
\label{sec:acknowledgments}
\input{Sections/Ackonwledge}
{
    
    \small
    \bibliographystyle{ieeenat_fullname}
    \bibliography{main}
}


\clearpage
\input{Sections/Appendix.tex}

\end{document}

%% file: Sections/Abstract.tex
\begin{abstract}
     Covariance matrices have proven highly effective across many scientific fields. Since these matrices lie within the Symmetric Positive Definite (SPD) manifold—a Riemannian space with intrinsic non-Euclidean geometry, the primary challenge in representation learning is to respect this underlying geometric structure. Drawing inspiration from the success of Euclidean deep learning, researchers have developed neural networks on the SPD manifolds for more faithful covariance embedding learning. A notable advancement in this area is the implementation of Riemannian batch normalization (RBN), which has been shown to improve the performance of SPD network models. Nonetheless, the Riemannian metric beneath the existing RBN might fail to effectively deal with the ill-conditioned SPD matrices (ICSM), undermining the effectiveness of RBN. In contrast, the Bures-Wasserstein metric (BWM) demonstrates superior performance for ill-conditioning. In addition, the recently introduced Generalized BWM (GBWM) parameterizes the vanilla BWM via an SPD matrix, allowing for a more nuanced representation of vibrant geometries of the SPD manifold. Therefore, we propose a novel RBN algorithm based on the GBW geometry, incorporating a learnable metric parameter. Moreover, the deformation of GBWM by matrix power is also introduced to further enhance the representational capacity of GBWM-based RBN. Experimental results on different datasets validate the effectiveness of our proposed method. The code is available at \url{https://github.com/jjscc/GBWBN}.
\end{abstract}

%% file: Sections/Introduction.tex
Covariance matrices are prevalent in numerous scientific fields. Due to their capability to capture high-order spatiotemporal fluctuations in data, they have proven to be highly effective 
in various applications, including Brain-Computer Interfaces (BCI) \citep{bci,kobler2022spd, matt}, Magnetic Resonance Imaging (MRI) \cite{PenX,dai,manifoldnet}, drone recognition \cite{spdnetbn,chen2023riemannian,chen2024liebn}, node classification \cite{nguyen2021geomnet,chen2021hybrid, RResNet,chen2024rmlr}, 
and action recognition \cite{spdnet,jdrml, chen2024adaptive,nguyen2024matrix,dspdnet,smdml,chen2025gyrogroup}. However, the geometry of non-singular covariance matrices doesn't conform to that of a vector space; instead, it forms a Riemannian manifold, known as the SPD manifold. This inherent difference prevents the direct 
application of Euclidean learning methods to such data. To bridge this gap, several Riemannian metrics have been developed, such as the Log-Euclidean Metric (LEM) \cite{ArsV}, Log-Cholesky Metric (LCM) \cite{lin2019riemannian}, and Affine Invariant Metric (AIM) \cite{PenX}, extending Euclidean computations to the SPD manifold.

Deep Neural Networks (DNNs)~\cite{resnet,vgg} have achieved remarkable breakthroughs over traditional shallow learning models in the field of pattern recognition and computer vision, primarily due to their ability to learn deep, nonlinear representations. Building on this success, some researchers have generalized the ideology of DNNs to Riemannian manifolds. Notably, SPDNet \cite{spdnet} stands out as a classical Riemannian neural network, providing a multi-stage, end-to-end nonlinear learning framework specifically designed for SPD matrices. In addition, numerous Euclidean-based model architectures, including Recurrent Neural Network (RNN)~\cite{chakraborty2018statistical,nguyen2022gyro}, Convolutional Neural Network (CNN)~\cite{manifoldnet}, and Transformers~\cite{ yang2024hypformer}, have been successfully adapted to operate on manifolds. In parallel, researchers have also investigated the fundamental building blocks necessary for Riemannian networks, such as 
pooling \cite{symnet,chen2025understanding}, 
classification \cite{spdmlr,nguyen2023building}, and transformation \cite{nguyen2023building,nguyen2024matrix}.

Traditional Batch Normalization (BatchNorm) has been instrumental in facilitating network training \cite{ioffe2015batch}, which motivates the development of Riemannian BatchNorm (RBN) over the SPD manifold. Among them, \citet{spdnetbn} introduced RBN for SPD networks, utilizing the AIM geometry. \citet{kobler2022spd} further advanced this methodology by extending control over the Riemannian variance. In parallel, \citet{chakraborty2020manifoldnorm} proposed two RBN frameworks and instantiated them on the SPD manifold under LEM and AIM geometries by incorporating first- and second-order statistics. More recently, \citet{chen2024liebn} developed the RBN over Lie groups, specifically showcasing their framework on the SPD manifold under the AIM, LEM, and LCM geometries. 

However, the latent metric geometries beneath the existing RBN might be ineffective in dealing with ill-conditioned SPD matrices (ICSM), which are yet prevalent in SPD modeling. Specifically, as covariance matrices are positive semi-definite, the following regularization is widely adopted: $\mathbf{X} \gets \mathbf{X} + \lambda \mathbf{I}_d$, where $\mathbf{I}_d$ is a $d$-by-$d$ identity matrix and $\lambda$ denotes a perturbation hyperparameter. Although such a strategy ensures the positive definite property of $\mathbf{X}$, it offers only limited improvement in alleviating ill-conditioning. For any SPD matrix $\mathbf{X}$, the condition number can be characterized as $\kappa(\mathbf{X})=\nicefrac{\mathrm{eig}_{\max }}{\mathrm{eig}_{\min }}$, where $\mathrm{eig}_{\max }$ and $\mathrm{eig}_{\min }$ denotes its maximal and minimal eigenvalues. Therefore, the post-regularization SPD matrix might be ill-conditioned. \textbf{\cref{lambda_eigenvalue_teaser} summarizes the ill-conditioning observed across all three used datasets, highlighting its widespread presence}.

\citet{han2021riemannian} show that the AIM exhibits a quadratic dependence on SPD matrices, which hinders effective learning over ICSM. In contrast, the BWM \cite{bhatia2019bures} has a linear dependence, offering advantages for ill-conditioning. 
Thereby, we resort to BWM for RBN. Besides, \citet{han2023learning} further updated BWM into a generalized version, named GBWM, by an SPD parameter. The GBWM is correlated with the AIM, as it is locally AIM \cite{han2023learning}. By setting the SPD parameter in GWBM to be learnable, the proposed BWM-based RBN algorithm is further extended into the one based on the GBWM. Compared with the existing methods, our method can not only better deal with ICSM, but also adapts to the evolving latent SPD geometry through the dynamic geometric parameter. Besides, we use the deformation of matrix power \cite{thanwerdas2022geometry} to further improve the representational capacity of the suggested GBWM-based RBN. For simplicity, we refer to the proposed method as GBWBN. \cref{summary} summarizes the existing RBN methods compared with ours. The effectiveness of our GBWBN is validated across three different benchmarking datasets involving two signal classification tasks, \textit{i.e.}, skeleton-based action recognition and EEG classification. In summary, our main contributions are as follows:

\begin{table}[!t]
    \centering
    \resizebox{0.85\linewidth}{!}{
    \begin{tabular}{ccc}
    \toprule
    Dataset & $\kappa>10^{3}$  & $\kappa>10^{5}$  \\
    \midrule
    MAMEM-SSVEP-II & 113 (22.6\%) & 11 (2.2\%) \\
    HDM05 & 2086 (100\%) & 2086 (100\%) \\
    NTU RGB+D & 56,880 (100\%) & 56,880 (100\%) \\
    \bottomrule
    \end{tabular}}
    \caption{
    Statistics (number and percentage) on the condition number ($\kappa = {\rm{eig}}_{\rm{max}} / {\rm{eig}}_{\rm{min}}$) of the SPD features without the RBN layer on different datasets. The statistics are summarized in the final epoch. This is a teaser table from \cref{combined_lambda_eigenvalue}.
    }
    \label{lambda_eigenvalue_teaser}
\end{table}

\begin{table}[!t]
    \centering
    \resizebox{1.0\linewidth}{!}{
    \begin{tabular} {c c c}
    \toprule
    \text { Methods } & \text { Involved Statistics}  & \text { Geometries } \\
    \midrule
    \text { SPDBN \cite{spdnetbn} } & \text { M } & \text { AIM } \\
    \text { SPDDSMBN \cite{kobler2022spd}} & \text { M+V } &  \text { AIM } \\
    \text { ManifoldNorm \cite[Algs. 1-2]{chakraborty2020manifoldnorm} } & \text { M+V } &  \text { AIM and LEM } \\
    \text { LieBN } \cite{chen2024liebn} & \text { M+V } & \text { Power-deformed AIM, LEM, and LCM } \\
    \midrule
    \text { GBWBN (Ours) } & \text { M+V } & \text { Power-deformed GBWM } \\
    \bottomrule
    \end{tabular}
    }
    \caption{Summary of several representative RBN methods on the SPD manifold, where M and V represent the mean and variance.}
    \label{summary}
\end{table}

\begin{itemize}
    \item[$\bullet$] \textbf{Challenging ICSM, effective countermeasure:} A new RBN algorithm based on the BW geometry—the first work capable of effectively addressing ICSM learning;
    \item[$\bullet$] \textbf{Learnable normalization, flexible method:} Our BWM-based RBN is extended to a learnable version incorporated with an SPD parameter and a matrix power-based nonlinear deformation of the metric geometry, aiming to generate a flexible and refined normalization space;
    \item [$\bullet$] \textbf{Experimental validity, functional pluggability:} Extensive empirical evaluations of our GBWBN on three benchmarking datasets under different SPD backbone networks.
\end{itemize}

%% file: Sections/Preliminaries.tex
The set of SPD matrices, denoted as 
$\boldsymbol{{\mathcal{S}}_{++}^d}:= \{\boldsymbol{\mathrm{X}}\in\mathbb{R}^{d\times d},\boldsymbol{\mathrm{X}}=\boldsymbol{\mathrm{X}}^{\rm{T}},\boldsymbol{v}^{\rm{T}}\boldsymbol{\mathrm{X}v}>0,\forall \boldsymbol{v}\in\mathbb{R}^d\backslash\{0_d\}\}$, forms a smooth manifold, known as the SPD manifold.
The tangent space at $\boldsymbol{\mathrm{X}}$ is denoted as 
${T}_{\boldsymbol{\mathrm{X}}}\boldsymbol{{\mathcal{S}}_{++}^d} 
\cong \boldsymbol{{\mathcal{S}}^d}$, where $\boldsymbol{{\mathcal{S}}^d}$ is the Euclidean space of real symmetric matrices.
Given $\boldsymbol{\mathrm{X}} \in  \boldsymbol{{\mathcal{S}}_{++}^d}$, the Riemannian metric at $\boldsymbol{\mathrm{X}}$ is denoted as: $g_{\boldsymbol{\mathrm{X}}}:{T}_{\boldsymbol{\mathrm{X}}}\boldsymbol{{\mathcal{S}}_{++}^d}\times {T}_{\boldsymbol{\mathrm{X}}}\boldsymbol{{\mathcal{S}}_{++}^d}\to \mathbb{R}$, which is often written as an inner product $\langle\cdot ,\cdot\rangle_{\boldsymbol{\mathrm{X}}}$. The induced norm of a tangent vector $\boldsymbol{\mathrm{S}}\in {T}_{\boldsymbol{\mathrm{X}}}\boldsymbol{{\mathcal{S}}_{++}^d}$ is given by $\left \| \boldsymbol{\mathrm{S}} \right \| _{\boldsymbol{\mathrm{X}}} = \sqrt{\langle \boldsymbol{\mathrm{S}},\boldsymbol{\mathrm{S}}\rangle_{\boldsymbol{\mathrm{X}}}}$.
Given $\boldsymbol{\mathrm{X}}_{1}, \boldsymbol{\mathrm{X}}_{2} \in \boldsymbol{{\mathcal{S}}_{++}^d} $, $\boldsymbol{\mathrm{S}}_{1}, \boldsymbol{\mathrm{S}}_{2} \in{T}_{\boldsymbol{\mathrm{X}}_{1}}\boldsymbol{{\mathcal{S}}_{++}^d}$, the Riemannian operators under BWM are shown in \cref{opera}. Wherein, the $\mathrm{tr}(\boldsymbol{\mathrm{S}})$ denotes the matrix trace of $\boldsymbol{\mathrm{S}}$, $\mathrm{vec}({\boldsymbol{\mathrm{S}}}_{1})$ and $\mathrm{vec}({\boldsymbol{\mathrm{S}}_{2}})$ are the vectorizations of $\boldsymbol{\mathrm{S}}_{1}$ and $\boldsymbol{\mathrm{S}}_{2}$, $\boldsymbol{\mathrm{\lambda}}_{i},\boldsymbol{\mathrm{\lambda}}_{i},\boldsymbol{\mathrm{\delta}}_{i},\boldsymbol{\mathrm{\delta}}_{j}$ represent the \textit{i}-th and \textit{j}-th eigenvalues obtained by SVD of $\boldsymbol{\mathrm{X}}_{1}$ and $\boldsymbol{\mathrm{X}}_{2}$ respectively.
$\boldsymbol{\mathrm{X}}\oplus\boldsymbol{\mathrm{X}}=\boldsymbol{\mathrm{X}}\otimes \boldsymbol{\mathrm{I}}+\boldsymbol{\mathrm{I}}\otimes \boldsymbol{\mathrm{X}}$ denotes the Kronecker sum and $\mathcal{L}_{\boldsymbol{\mathrm{X}}_{1}}(\boldsymbol{\mathrm{S}}_{1})$ denotes the Lyapunov operator, \textit{i.e.,} ${\boldsymbol{\mathrm{X}}_{1}\mathcal{L}_{\boldsymbol{\mathrm{X}}_{1}}(\boldsymbol{\mathrm{S}}_{1})}+{\mathcal{L}_{\boldsymbol{\mathrm{X}}_{1}}(\boldsymbol{\mathrm{S}}_{1})\boldsymbol{\mathrm{X}}_{1}}={\boldsymbol{\mathrm{S}}_{1}}$.
The Lyapunov operator can be solved by eigendecomposition~\cite{han2021riemannian}:
\begin{equation}
\label{ly_solve}
    \mathcal{L}_{\boldsymbol{\mathrm{X}}_{1}}({\boldsymbol{\mathrm{S}}}_{1})= {\boldsymbol{\mathrm{V}}}\left[\frac{{\boldsymbol{\mathrm{S}'_{1}}}_{i,j}}{{\boldsymbol{\mathrm{\delta}}}_{i}+{\boldsymbol{\mathrm{\delta}}}_{j}}\right]_{i,j}{\boldsymbol{\mathrm{V}}^{T}},
\end{equation}
where $\boldsymbol{\mathrm{S'}}_{1} = \boldsymbol{\mathrm{V}}^{T} \boldsymbol{\mathrm{S}}_{1} \boldsymbol{\mathrm{V}}$, $\boldsymbol{\mathrm{X}}_{1} = \boldsymbol{\mathrm{V}}\boldsymbol{\mathrm{\Delta}}\boldsymbol{\mathrm{V}}^{T}$ is the eigendecomposition of $\boldsymbol{\mathrm{X}}_{1}\in\boldsymbol{\mathcal{S}_{++}^d}$.

Although $\boldsymbol{\mathrm{X}}_{1}\boldsymbol{\mathrm{X}}_{2}$ is not SPD, 
the following equation of $(\boldsymbol{\mathrm{X}}_{1}\boldsymbol{\mathrm{X}}_{2})^{\frac{1}{2}}$ makes each (inverse) square root can be computed using SVD \cite{minh2022alpha}.
\begin{equation}
\label{sqr}
(\boldsymbol{\mathrm{X}}_{1}\boldsymbol{\mathrm{X}}_{2})^{\frac{1}{2}}=\boldsymbol{\mathrm{X}}_{1}^{\frac{1}{2}}(\boldsymbol{\mathrm{X}}_{1}^{\frac{1}{2}}\boldsymbol{\mathrm{X}}_{2}\boldsymbol{\mathrm{X}}_{1}^{\frac{1}{2}})^{\frac{1}{2}}\boldsymbol{\mathrm{X}}_{1}^{-\frac{1}{2}}.
\end{equation}

The geodesic is not complete on the Riemannian manifold $(\boldsymbol{{\mathcal{S}}_{++}^d}{g}^{\mathrm{BW}})$ \cite{malago2018wasserstein}. In \cref{opera},  the variable $t$ in the geodesic depends on as follows: $\boldsymbol{\mathrm{I}}_{d}+t\mathcal{L}_{\boldsymbol{\mathrm{X}}_{1}}(\boldsymbol{\mathrm{S}}_{1})\in\boldsymbol{{\mathcal{S}}_{++}^d}$. For parallel transport, $\boldsymbol{\mathrm{X}}_{1},\boldsymbol{\mathrm{X}}_{2}$ must be commutative matrices \cite{thanwerdas2023n}.

For any two $\boldsymbol{\mathrm{S}}_{1}, \boldsymbol{\mathrm{S}}_{2} \in T_{\boldsymbol{\mathrm{X}}_{1}}{\boldsymbol{{\mathcal{S}}_{++}^d}}$, the AIM on the SPD manifold can be rewritten as:
\begin{equation}
\begin{split}
    \label{eq:metric_aim}
    {g}_{\boldsymbol{\mathrm{X}}_{1}}^{\mathrm{AI}}(\boldsymbol{\mathrm{S}}_{1},\boldsymbol{\mathrm{S}}_{2})&=\mathrm{tr}(\boldsymbol{\mathrm{X}}_{1}^{-1}\boldsymbol{\mathrm{S}}_{1}\boldsymbol{\mathrm{X}}_{1}^{-1}\boldsymbol{\mathrm{S}}_{2})\\&=\mathrm{vec}(\boldsymbol{\mathrm{S}}_{1})^{T}(\boldsymbol{\mathrm{X}}_{1}\otimes \boldsymbol{\mathrm{X}}_{1})^{-1}\mathrm{vec}(\boldsymbol{\mathrm{S}}_{2}),
\end{split}
\end{equation}
where $\otimes$ denotes the Kronecker product. Comparing the 
Riemannian metric based on the BWM in \cref{opera} and \cref{eq:metric_aim}, the AIM has a quadratic dependence on $\boldsymbol{\mathrm{X}}_{1}$ while the BWM has a linear dependence \cite{han2021riemannian}, revealing the BWM is more suitable for learning ICSM.

For any two $\boldsymbol{\mathrm{S}}_{1}, \boldsymbol{\mathrm{S}}_{2} \in T_{\boldsymbol{\mathrm{X}}}{\boldsymbol{{\mathcal{S}}_{++}^d}}$, the GBWM can be written as \cite{han2023learning}
{\small
\begin{align}
    \label{eq:metric_gbwm}
    {g}_{\boldsymbol{\mathrm{X}}}^{\mathrm{GBW}}(\boldsymbol{\mathrm{S}}_{1},\boldsymbol{\mathrm{S}}_{2})&=\frac{1}{2}{\mathrm{tr}}\left(\mathcal{L}_{\boldsymbol{\mathrm{X}},\boldsymbol{\mathrm{M}}}\left(\boldsymbol{\mathrm{S}}_{1}\right)\boldsymbol{\mathrm{S}}_{2}\right) \\&=\frac{1}{2}\mathrm{vec}(\boldsymbol{\mathrm{S}}_{1})^{T}(\boldsymbol{\mathrm{X}}\otimes\boldsymbol{\mathrm{M}}+\boldsymbol{\mathrm{M}}\otimes\boldsymbol{\mathrm{X}})^{-1}\mathrm{vec}(\boldsymbol{\mathrm{S}}_{2}), \nonumber
\end{align}
}where $\boldsymbol{\mathrm{M}}\in\boldsymbol{{\mathcal{S}}_{++}^d}$, $\mathcal{L}_{\boldsymbol{\mathrm{X}},\boldsymbol{\mathrm{M}}}(\boldsymbol{\mathrm{S}})$ is the generalized Lyapunov operator, which is the solution to the matrix equation ${\boldsymbol{\mathrm{X}}\mathcal{L}_{\boldsymbol{\mathrm{X}},\boldsymbol{\mathrm{M}}}(\boldsymbol{\mathrm{S}})\boldsymbol{\mathrm{M}}}+\boldsymbol{\mathrm{M}}{\mathcal{L}_{\boldsymbol{\mathrm{X}},\boldsymbol{\mathrm{M}}}(\boldsymbol{\mathrm{S}})\boldsymbol{\mathrm{X}}}={\boldsymbol{\mathrm{S}}}, \boldsymbol{\mathrm{S}}\in\boldsymbol{{\mathcal{S}}^d}$. Choosing $\boldsymbol{\mathrm{M}}$ is equivalent to selecting a suitable metric. When $\boldsymbol{\mathrm{M}}=\boldsymbol{\mathrm{I}}$, it reduces to the BWM. When $\boldsymbol{\mathrm{M}}=\boldsymbol{\mathrm{X}}$, it coincides locally with the AIM \cite{han2023learning}.
Consequently, it connects the BWM and AIM (locally). 

\begin{table}[!t]
    \centering
    \resizebox{0.99\linewidth}{!}{
    \begin{tabular}{c}
    \toprule
    Riemannian operators under the BWM \\
    \midrule
    \begin{tcolorbox}[colback=gray!20, boxrule=0pt, colframe=white, halign=center, height=0.5cm, valign=center]
    \textbf{Riemannian metric} 
    \end{tcolorbox} \\
    \begin{minipage}{\linewidth}
    $
    \begin{array}{rcl}
        {g}_{\boldsymbol{\mathrm{X}}_{1}}^{\mathrm{BW}}(\boldsymbol{\mathrm{S}}_{1},\boldsymbol{\mathrm{S}}_{2}) & = & \frac{1}{2}\mathrm{tr}\left(\mathcal{L}_{\boldsymbol{\mathrm{X}}_{1}}(\boldsymbol{\mathrm{S}}_{1})\boldsymbol{\mathrm{S}}_{2}\right) \\
        & = & \frac{1}{2}\mathrm{vec}(\boldsymbol{\mathrm{S}}_{1})^{T}(\boldsymbol{\mathrm{X}}_{1} \oplus \boldsymbol{\mathrm{X}}_{1})^{-1}\mathrm{vec}(\boldsymbol{\mathrm{S}}_{2})
    \end{array}
    $
    \end{minipage} \\
    \begin{tcolorbox}[colback=gray!20, boxrule=0pt, colframe=white, halign=center, height=0.5cm, valign=center]
    \textbf{Riemannian distance} 
    \end{tcolorbox} \\
    \begin{minipage}{\linewidth}
    $d_{\mathrm{BW}}^{2}(\boldsymbol{\mathrm{X}}_{1},\boldsymbol{\mathrm{X}}_{2}) = \mathrm{tr}({\boldsymbol{\mathrm{X}}_{1}}) + \mathrm{tr}({\boldsymbol{\mathrm{X}}_{2}}) - 2\left(\mathrm{tr}\left(\boldsymbol{\mathrm{X}}_{1}^{\frac{1}{2}} \boldsymbol{\mathrm{X}}_{2} \boldsymbol{\mathrm{X}}_{1}^{\frac{1}{2}}\right)\right)^{\frac{1}{2}}$
    \end{minipage} \\
    \begin{tcolorbox}[colback=gray!20, boxrule=0pt, colframe=white, halign=center, height=0.5cm, valign=center]
    \textbf{Riemannian geodesic}
    \end{tcolorbox} \\
    \begin{minipage}{\linewidth}
    \centering
    $\gamma_{\boldsymbol{\mathrm{X}}_{1},\boldsymbol{\mathrm{S}}_{1}}(t) = \boldsymbol{\mathrm{X}}_{1} + t\boldsymbol{\mathrm{S}}_{1} + t^2 \mathcal{L}_{\boldsymbol{\mathrm{X}}_{1}}(\boldsymbol{\mathrm{S}}_{1}) \boldsymbol{\mathrm{X}}_{1} \mathcal{L}_{\boldsymbol{\mathrm{X}}_{1}}(\boldsymbol{\mathrm{S}}_{1})$
    \end{minipage} \\
    \begin{tcolorbox}[colback=gray!20, boxrule=0pt, colframe=white, halign=center, height=0.5cm, valign=center]
    \textbf{Riemannian exponential map}
    \end{tcolorbox} \\
    \begin{minipage}{\linewidth}
    \centering
    $\mathrm{Exp}_{\boldsymbol{\mathrm{X}}_{1}}\boldsymbol{\mathrm{S}}_{1} = \boldsymbol{\mathrm{X}}_{1} + \boldsymbol{\mathrm{S}}_{1} + \mathcal{L}_{\boldsymbol{\mathrm{X}}_{1}}(\boldsymbol{\mathrm{S}}_{1}) \boldsymbol{\mathrm{X}}_{1} \mathcal{L}_{\boldsymbol{\mathrm{X}}_{1}}(\boldsymbol{\mathrm{S}}_{1})$
    \end{minipage} \\
    \begin{tcolorbox}[colback=gray!20, boxrule=0pt, colframe=white, halign=center, height=0.5cm, valign=center]
    \textbf{Riemannian logarithmic map}
    \end{tcolorbox} \\
    \begin{minipage}{\linewidth}
    \centering
    $\mathrm{Log}_{\boldsymbol{\mathrm{X}}_{1}}\boldsymbol{\mathrm{X}}_{2} = (\boldsymbol{\mathrm{X}}_{2} \boldsymbol{\mathrm{X}}_{1})^{\frac{1}{2}} + (\boldsymbol{\mathrm{X}}_{1} \boldsymbol{\mathrm{X}}_{2})^{\frac{1}{2}} - 2 \boldsymbol{\mathrm{X}}_{1}$
    \end{minipage} \\
    \begin{tcolorbox}[colback=gray!20, boxrule=0pt, colframe=white, halign=center, height=0.5cm, valign=center]
    \textbf{Parallel Transport} 
    \end{tcolorbox} \\
    \begin{minipage}{\linewidth}
    \centering
    $\Gamma _{\boldsymbol{\mathrm{X}}_{1}\to \boldsymbol{\mathrm{X}}_{2}}(\boldsymbol{\mathrm{S}})
    =\boldsymbol{\mathrm{U}}\left [ \sqrt{\frac{\boldsymbol{\mathrm{\delta}}_{i}+\boldsymbol{\mathrm{\delta}}_{j}}{\boldsymbol{\mathrm{\lambda }}_{i}+\boldsymbol{\mathrm{\lambda}}_{j}} }\left [ \boldsymbol{\mathrm{S'}}  \right ]_{i,j}\right]_{i,j}\boldsymbol{\mathrm{U}}^{T}$
    \end{minipage} \\
    \bottomrule
    \end{tabular}
    }
    \caption{Basic operators based on the BWM \cite{bhatia2019bures}.}
    \label{opera}
\end{table}

%% file: Sections/Proposed_method.tex
In this section, we begin by introducing RBN based on the BWM, from which the GBWM-based batch normalization is derived. This foundation enables a smooth extension of RBN to integrate the broader capabilities of GBWM.
Due to page limitation, all the proofs are placed in \cref{proof:main}.

\subsection{RBN based on the BWM}
This subsection begins with an introduction to the BWM-based Riemannian mean. We then generalize the parallel transport, extending it from tangent vectors to data points on the SPD manifold. Following this, we detail the data scaling process. Finally, the aforementioned components are integrated to achieve manifold-valued batch normalization.

\subsubsection{Riemannian mean and variance}
Let $\boldsymbol{\mathrm{X}}_1,\dots,\boldsymbol{\mathrm{X}}_N$ denote $N$ points on the SPD manifold, $\boldsymbol{{\Omega}}=[{\omega}_1,\dots,{\omega}_N]$ be a weight vector, satisfying ${\omega }_{i}\ge {0}$ and $\sum_i {\omega}_{i}= 1$. Denoting the weighted Fréchet mean (WFM) as $\mathbb{E}_{N}^{\boldsymbol{\Omega}}{(\{\boldsymbol{\mathrm{X}}_{i}\}_{i\le {N}})}$, the Riemannian barycenter, signified as $\boldsymbol{\mathcal{B}}$, we can have the following:  
\begin{equation}
\label{eq:mean}
\begin{split}
   \boldsymbol{\mathcal{B}}=\mathbb{E}_{N}^{\boldsymbol{\Omega}}{(\{\boldsymbol{\mathrm{X}}_{i}\}_{i\le {N}})} = \underset{\boldsymbol{\mathrm{G}}\in\boldsymbol{\mathcal{S}}_{++}^d}{\mathrm{arg\,min}}\sum_{i=1}^{N}\omega_i{d_{\mathrm{BW}}^{2}}(\boldsymbol{\mathrm{G}},\boldsymbol{\mathrm{X}}_{i}).
\end{split}
\end{equation}
When $\omega_{i}=\frac{1}{N}$, \cref{eq:mean} simplifies to the Fréchet mean(FM). The Fréchet variance $\boldsymbol{\mathrm{\nu}}^{2}$ is the value achieved at the minimizer of the Fréchet mean:
\begin{equation}
\label{eq:var}
\begin{split}
    \boldsymbol{\mathrm{\nu}}^{2}=\frac{1}{N}\sum_{i=1}^{N}{d_{\mathrm{BW}}^{2}}(\boldsymbol{\mathcal{B}},\boldsymbol{\mathrm{X}}_{i}).
\end{split}
\end{equation}

In this paper, we use the terms Fréchet mean signifying Riemannian mean, and Fréchet variance signifying Riemannian variance. 
When ${N}={2}$, \textit{i.e.,} $\boldsymbol{\Omega}=(1-\omega,\omega), \omega\in[0,1]$, a closed-formed solution of \cref{eq:mean} exists, corresponding to the geodesic connecting $\boldsymbol{\mathrm{X}}_{1}$ and $\boldsymbol{\mathrm{X}}_{2}$:
\begin{equation}
\label{eq:mean_compute1}
\begin{split}
    \mathbb{E}_2^{\boldsymbol{\Omega}}(\boldsymbol{\mathrm{X}}_{1},\boldsymbol{\mathrm{X}}_{2})&=(1-\omega)^{2}\boldsymbol{\mathrm{X}}_{1}+  {\omega}^{2}\boldsymbol{\mathrm{X}}_{2} + \\
    &\omega(1-\omega)\left((\boldsymbol{\mathrm{X}}_{2}\boldsymbol{\mathrm{X}}_{1})^{\frac{1}{2}}+(\boldsymbol{\mathrm{X}}_{1}\boldsymbol{\mathrm{X}}_{2})^{\frac{1}{2}}\right).
\end{split}
\end{equation}

As shown in \cite{bhatia2019bures}, the minimized solution to \cref{eq:mean} is unique, which is also a solution to a non-linear equation:
\begin{equation}
\label{eq:mean1}
\begin{split}
    \boldsymbol{\mathrm{G}}=\sum_{i=1}^{N}\omega_{i}(\boldsymbol{\mathrm{G}}^{\frac{1}{2}}\boldsymbol{\mathrm{X}}_{i}\boldsymbol{\mathrm{G}}^{\frac{1}{2}})^{\frac{1}{2}}.
\end{split}
\end{equation}

When $N > 2$, the solution to \cref{eq:mean1} can be obtained by a novel iterative method, called fixed-point iteration ~\cite{bhatia2019bures}:
\begin{equation}
\label{eq:mean_compute2}
\begin{split}
    \mathrm{H}(\boldsymbol{\mathrm{G}})=\boldsymbol{\mathrm{G}}^{-\frac{1}{2}}\left(\sum_{i=1}^{N}\omega_{i}(\boldsymbol{\mathrm{G}}^{\frac{1}{2}}\boldsymbol{\mathrm{X}}_{i}\boldsymbol{\mathrm{G}}^{\frac{1}{2}})^{\frac{1}{2}}\right)^{2}\boldsymbol{\mathrm{G}}^{-\frac{1}{2}}.
\end{split}
\end{equation}
Its update role is $\boldsymbol{\mathrm{G}}_{t+1}=\mathrm{H}(\boldsymbol{\mathrm{G}}_{t})$ and converges to $\boldsymbol{\mathcal{B}}$, \textit{i.e.}, $\lim_{t \to \infty} \boldsymbol{\mathrm{G}}_{t}=\boldsymbol{\mathcal{B}}$. Considering that the batch barycenter is a noisy estimation of the true barycenter, a relaxed approximation is sufficient ~\cite{spdnetbn,chen2024liebn}. Thereby, we set the number of iterations to one for computational efficiency. Furthermore, although $\boldsymbol{\mathrm{G}}_{t}$ can be initialized as any arbitrary SPD matrix, it is reasonable to endow it with the arithmetic mean.

\subsubsection{Centering data via parallel transport}
The conventional batch normalization in Euclidean space typically involves centering and biasing a set of samples using vector subtraction and addition. In contrast, there is no vector structure for a curved Riemannian manifold. Given a batch of data points on the SPD manifold, a viable and effective method to relocate features around the batch mean $\boldsymbol{\mathcal{B}}$ or bias them towards a parameter $\boldsymbol{\mathcal{G}}$ is to employ parallel transport (PT).

Drawing inspiration from~\cite{spdnetbn}, we extend the tangent space PT given in \cref{opera} to its manifold counterpart. Let commuting matrices $\boldsymbol{\mathrm{X}}_{1},\boldsymbol{\mathrm{X}}_{2}\in \boldsymbol{{\mathcal{S}}_{++}^d}$, for any $\boldsymbol{\mathrm{P}}\in\boldsymbol{{\mathcal{S}}_{++}^d}$, the PT of $\boldsymbol{\mathrm{P}}$ from $\boldsymbol{\mathrm{X}}_{1}$ to $\boldsymbol{\mathrm{X}}_{2}$ can be realized via the following manifold map $\phi :\boldsymbol{{\mathcal{S}}_{++}^d}\to \boldsymbol{{\mathcal{S}}_{++}^d}$:
    \begin{equation}
    \begin{split}
        \phi(\boldsymbol{\mathrm{P}})=\mathrm{Exp}_{\boldsymbol{\mathrm{X}}_{2}}\left(\Gamma_{\boldsymbol{\mathrm{X}_{1}}\to \boldsymbol{\mathrm{X}}_{2}}\left(\mathrm{Log}_{\boldsymbol{\mathrm{X}}_{1}}(\boldsymbol{\mathrm{P}})\right)\right).
    \end{split}
    \end{equation}

\subsubsection{Batch normalization}
In deep learning, batch normalization is a widely used technique to speed up convergence and mitigate internal covariate shifts \cite{ioffe2015batch}. The standard batch normalization layer applies slightly different transformations to the data during the training and testing. In the training phase, normalization is performed using 
the batch mean $\boldsymbol{\mathcal{B}}_{b}$ and variance $\boldsymbol{\mathrm{\nu}}_{b}^{2}$. Meanwhile, the running statistics—mean $\boldsymbol{\mathcal{B}}_{r}$ and variance $\boldsymbol{\mathrm{\nu}}_{r}^{2}$—are  
iteratively updated, where they are applied for normalization during testing. 

Given a batch of samples $\{\boldsymbol{\mathrm{X}}_i\}_{i=1}^{N}$, the batch normalization based on the BWM is defined as follows:

\textbullet$\,$Data centralizing on the SPD manifold, \textit{i.e.}, shifting $\{\boldsymbol{\mathrm{X}}_i\}_{i=1}^{N}$ with Riemannian mean $\boldsymbol{\mathcal{B}}$ to the identity matrix $\boldsymbol{\mathrm{I}}_{d}$ is given by
\begin{equation}
\label{eq:center}
    \bar{\boldsymbol{\mathrm{X}}}_{i}=\phi_{\boldsymbol{\mathcal{B}}\to \boldsymbol{\mathrm{I}}_{d}}(\boldsymbol{\mathrm{X}}_{i})=\mathrm{Exp}_{\boldsymbol{\mathrm{I}}_{d}}\left(\Gamma_{\boldsymbol{\mathcal{B}}\to \boldsymbol{\mathrm{I}}_{d}}\left(\mathrm{Log}_{\boldsymbol{\mathcal{B}}}\left(\boldsymbol{\mathrm{X}}_{i}\right)\right)\right).
\end{equation}

\textbullet$\,$Then, data scaling can be defined as follows:
\begin{equation}
\label{eq:data scaling}
\begin{split}
    \check{\boldsymbol{\mathrm{X}}}_{i}=\mathrm{Exp}_{\boldsymbol{\mathrm{I}}_{d}}(\frac{\boldsymbol{\mathrm{s}}}{\sqrt{{\boldsymbol{\mathrm{\nu}}}^{2}+{\boldsymbol{\mathrm{\epsilon}}}}}(\mathrm{Log}_{\boldsymbol{\mathrm{I}}_{d}}(\bar{\boldsymbol{\mathrm{X}}}_{i}))),
\end{split}
\end{equation}
where $\boldsymbol{\mathrm{\nu}}^{2}$ denotes the Fréchet variance, $\boldsymbol{\mathrm{s}} \in\mathbb{R}\backslash\{0 \}$ is a scaling factor, and $\boldsymbol{\mathrm{\epsilon}}$ is a small positive constant.

\textbullet$\,$Finally, data biasing is conducted on the SPD manifold, transferring a set of samples from the point $\boldsymbol{\mathrm{I}}_{d}$ to the target point $\boldsymbol{\mathcal{G}}$, as detailed below:
\begin{equation}
\label{eq:bias}
\Tilde{\boldsymbol{\mathrm{X}}}_{i}=\phi_{\boldsymbol{\mathrm{I}}_{d}\to \boldsymbol{\mathcal{G}}}(\check{\boldsymbol{\mathrm{X}}}_{i})=\mathrm{Exp}_{\boldsymbol{\mathcal{G}}}\left(\Gamma_{\boldsymbol{\mathrm{I}}_{d}\to \boldsymbol{\mathcal{G}}}\left(\mathrm{Log}_{\boldsymbol{\mathrm{I}}_{d}}(\check{\boldsymbol{\mathrm{X}}}_{i})\right)\right).
\end{equation}

\subsubsection{Data scaling}
On $\boldsymbol{{\mathcal{S}}_{++}^d}$, several notions of Gaussian density have been proposed \cite{barbaresco2019jean, said2017riemannian}. In this work, we adopt the Gaussian distribution on $\boldsymbol{{\mathcal{S}}_{++}^d}$ with a mean $\boldsymbol{\mathcal{B}}\in\boldsymbol{{\mathcal{S}}_{++}^d}$ and variance $\boldsymbol{\mathrm{\sigma}}^{2}$ from \cite{said2017riemannian}, which is defined as:
\begin{equation}
\label{eq:density}
\begin{split}
    p(\boldsymbol{\mathrm{X}} \mid \boldsymbol{\mathcal{B}
    },\boldsymbol{\mathrm{\sigma}}^{2})=\frac{1}{\zeta(\boldsymbol{\mathrm{\sigma}})}\mathrm{exp}\left(-\frac{d(\boldsymbol{\mathrm{X}},\boldsymbol{\mathcal{B}})^{2}}{2\boldsymbol{\mathrm{\sigma}}^{2}}\right),
\end{split}
\end{equation}
where $\zeta(\boldsymbol{\mathrm{\sigma}})$ is the normalization constant. 
\begin{proposition}
\label{props:control_var}
    Given $N$ SPD matrices $\{{\boldsymbol{\mathrm{X}}_{i}}\}_{i=1}^{N}$, $\boldsymbol{\mathrm{s}} \in\mathbb{R}\backslash\{0 \}$, defining $\boldsymbol{\mathrm{\psi}}_{\boldsymbol{\mathrm{s}}}(\boldsymbol{\mathrm{X}}_{i})=\mathrm{Exp}_{\boldsymbol{\mathrm{I}}_{d}}\left(\boldsymbol{\mathrm{s}}\mathrm{Log}_{\boldsymbol{\mathrm{I}}_{d}}(\boldsymbol{\mathrm{X}}_{i})\right)=\left(\boldsymbol{\mathrm{s}}(\boldsymbol{\mathrm{X}}^{\frac{1}{2}}-\boldsymbol{\mathrm{I}}_{d})+\boldsymbol{\mathrm{I}}_{d}\right)^{2}$, we control the dispersion from $\boldsymbol{\mathrm{I}}_{d}$:
    \begin{equation}
    \label{eq:control_var}
    \sum_{i=1}^{N}\omega_i{d_{\mathrm{BW}}}\left(\boldsymbol{\mathrm{\psi}}_{\boldsymbol{\mathrm{s}}}(\boldsymbol{\mathrm{X}}_{i}),\boldsymbol{\mathrm{I}}_{d}\right)=\boldsymbol{\mathrm{s}}\sum_{i=1}^{N}\omega_i{d_{\mathrm{BW}}}(\boldsymbol{\mathrm{X}}_{i},\boldsymbol{\mathrm{I}}_{d}),
    \end{equation}
where $\{\omega_{i}\}_{i=1}^N$ are weights satisfying ${\omega }_{i}\ge {0}$, $\sum_i {\omega}_{i}= 1$. 
\end{proposition}
\begin{proof}
    The proof is presented in \cref{app:proof-control_var}.
\end{proof}
Therefore, data scaling can be achieved by scaling data on the tangent space at the Riemannian mean.

\subsection{RBN based on the GBWM}

We generalize the BWM-based RBN and introduce a learnable version that implements RBN under the GBWM.

\subsubsection{Power-deformed GBWM}

Previous works have extended the AIM, LEM, LCM, and BWM to power-deformed metrics \cite{thanwerdas2019affine,thanwerdas2022geometry,chen2024liebn} using the diffeomorphism matrix power transformation $\phi_\theta(\boldsymbol{\mathrm{X}})=\boldsymbol{\mathrm{X}}^\theta, \forall \boldsymbol{\mathrm{X}} \in \boldsymbol{{\mathcal{S}}_{++}^d}$.
Intuitively, the matrix power parameter $\theta$ allows for interpolation between the original metric ($\theta = 1$) and a LEM-like metric ($\theta \rightarrow 0$) \cite{thanwerdas2022geometry}.
Inspired by the flexibility of the power function, we suggest the power-deformed GBWM, denoted as $(\theta)$-GBWM. For any $ \boldsymbol{\mathrm{X}} \in \boldsymbol{{\mathcal{S}}_{++}^d}$ and $\boldsymbol{\mathrm{S}}_{1}, \boldsymbol{\mathrm{S}}_{2} \in {T}_{\boldsymbol{\mathrm{X}}}\boldsymbol{{\mathcal{S}}_{++}^d},$ it can be formulated as: 
    \begin{equation}
    \begin{split}
    \label{eq:pow-metric}
        g^{(\theta)-\mathrm{GBW}}_{\boldsymbol{\mathrm{X}}} &\left(\boldsymbol{\mathrm{S}}_{1},\boldsymbol{\mathrm{S}}_{2} \right) = \\ 
        &\frac{1}{\theta^2} g_{\boldsymbol{\mathrm{X}}^\theta}^{\mathrm{GBW}} \left( \left(\phi_\theta\right)_{*,\boldsymbol{\mathrm{X}}} \left(\boldsymbol{\mathrm{S}}_{1}\right), \left(\phi_\theta\right)_{*,\boldsymbol{\mathrm{X}}} \left(\boldsymbol{\mathrm{S}}_{2}\right)\right),
    \end{split}
    \end{equation}
where $\phi_\theta$ signifies the matrix power, and $(\phi_\theta)_{*,\boldsymbol{\mathrm{X}}}(\cdot)$ denotes the differential map of $\phi_\theta$ at $\boldsymbol{\mathrm{X}}$.

\begin{proposition}[Deformation]
    \label{props:defromed_0}
    When $\theta \to 0$, $\forall$ $\boldsymbol{\mathrm{X}} \in {{\mathcal{S}}_{++}^d}$ and $\forall$ $\boldsymbol{\mathrm{S}} \in {T}_{\boldsymbol{\mathrm{X}}}{\boldsymbol{{\mathcal{S}}_{++}^d}}$, the following can be derived:
    {\small
    \begin{equation}
    \label{diformed_metrics_lim}
    g^{(\theta)-\mathrm{GBW}}_{\boldsymbol{\mathrm{X}}} \left(\boldsymbol{\mathrm{S}},\boldsymbol{\mathrm{S}}\right) \overset{\theta \to 0}{\longrightarrow}     
    \frac{1}{2} \left \langle \log_{*,\boldsymbol{\mathrm{X}}}\left(\mathcal{L}_{\boldsymbol{\mathrm{M}}}\left(\boldsymbol{\mathrm{S}}\right)\right),\log_{*,\boldsymbol{\mathrm{X}}}\left(\boldsymbol{\mathrm{S}}\right) \right \rangle, 
    \end{equation}
    }
    where the right side can be viewed as a variant of LEM.
\end{proposition}
\begin{proof}
    The proof is presented in \cref{app:proof-deformed_0}.
\end{proof}

The method studied in~\cite{thanwerdas2022geometry} extended AIM into power-deformed metric, denoted as $(\theta)$-AIM. 
We find that $(\theta)$-GBWM is locally power-deformed $(\theta)$-AIM, which is detailed in the following proposition.

\begin{proposition}[Locally Deformed AIM]
\label{props:local_deformed}
    For any $\boldsymbol{\mathrm{X}} \in {{\mathcal{S}}_{++}^d}$ and $\boldsymbol{\mathrm{S}}_{1}, \boldsymbol{\mathrm{S}}_{2} \in {T}_{\boldsymbol{\mathrm{X}}}{\boldsymbol{{\mathcal{S}}_{++}^d}}$, we have the following:
    \begin{equation}
    \label{eq:local_deform}
    g^{(\theta)-\mathrm{GBW}}_{\boldsymbol{\mathrm{X}}} \left(\boldsymbol{\mathrm{S}}_{1},\boldsymbol{\mathrm{S}}_{2}\right)=\frac{1}{4}g^{(\theta)-\mathrm{AI}}_{\boldsymbol{\mathrm{X}}}\left(\boldsymbol{\mathrm{S}}_{1},\boldsymbol{\mathrm{S}}_{2}\right).
    \end{equation}
\end{proposition}
\begin{proof}
    The proof is presented in \cref{app:proof-local_deform}.
\end{proof}
\cref{props:defromed_0,props:local_deformed} reveal that the power-deformed GBWM is a locally deformed AIM and LEM-like metric $(\theta \to 0)$. Therefore, the power-deformed GBWM-based RBN may achieve better results than the GBWM-based one. This is further demonstrated in the following experiments.

\begin{algorithm}[t]
  \footnotesize
  \SetAlgoNlRelativeSize{-1}
  \SetAlgoLined
  \setlength{\textfloatsep}{0.1cm}
  \setlength{\floatsep}{0.1cm}
  \setlength{\intextsep}{0.1cm}
  \SetInd{0.3em}{0.5em}
  \setlength{\parskip}{0pt}
  \caption{Batch normalization on the SPD manifold under deformed GBWM}
  \label{alg:bn}
  \SetKwInOut{Input}{Input}
  \SetKwInOut{Output}{Output}

  \Input{batches of $N$ SPD matrices $\{{\boldsymbol{\mathrm{X}}_{i}}\}_{i=1}^{N}$, small positive constant $\boldsymbol{\mathrm{\epsilon}}$, running mean $\boldsymbol{\mathcal{B}}_{r}=\boldsymbol{\mathrm{I}}_d$, running variance $\boldsymbol{\mathrm{\nu}}_{r}^{2}=1$, power parameter $\theta \in \mathbb{R}\backslash\{0\}$, momentum $\omega$, a learnable parameter $\boldsymbol{\mathrm{M}}\in \boldsymbol{{\mathcal{S}}_{++}^d}$, bias $\boldsymbol{\mathcal{G}}\in \boldsymbol{{\mathcal{S}}_{++}^d}$, and scaling parameter $\boldsymbol{\mathrm{s}}\in \mathbb{R}\setminus \{0\}$.}
  \Output{normalized batch $\{\boldsymbol{\mathrm{\Acute{X}}}_i\}_{i=1}^{N}$.}

  ${\boldsymbol{\mathrm{\hat{X}}}}_i = \boldsymbol{\mathrm{M}}^{-\frac{1}{2}}\boldsymbol{\mathrm{X}}_{i}^{\theta}\boldsymbol{\mathrm{M}}^{-\frac{1}{2}}$, $\;{\boldsymbol{\mathcal{\hat{G}}}}_i = \boldsymbol{\mathrm{M}}^{-\frac{1}{2}}\boldsymbol{\mathcal{G}}^{\theta}\boldsymbol{\mathrm{M}}^{-\frac{1}{2}}$ $\;\;\;\;\;$

  \BlankLine
  \If{\textit{training}}{
      $\boldsymbol{\mathcal{B}}_{b} \gets \mathbb{E}_{N}(\{\boldsymbol{\mathrm{\hat{X}}}_i\}_{i=1}^{N}),\boldsymbol{\mathrm{\nu}}_{b}^{2} \gets \frac{1}{N}\sum_{i=1}^{N}\frac{1}{\theta^{2}} {d_{\mathrm{BW}}^{2}}(\boldsymbol{\mathcal{B}}_{b}, \boldsymbol{\mathrm{\hat{X}}}_{i})$;  

      $\boldsymbol{\mathcal{B}}_{r}\gets \mathbb{E}_{2}^{\Omega}(\boldsymbol{\mathcal{B}}_{r},\boldsymbol{\mathcal{B}}_{b}), \boldsymbol{\mathrm{\nu}}_{r}^{2} \gets (1-\omega)\boldsymbol{\mathrm{\nu}}_{r}^{2} + \omega \boldsymbol{\mathrm{\nu}}_{b}^{2} $;  
      \\
  } 

  \textbf{if} {\textit{training}} \textbf{then} $\boldsymbol{\mathcal{B}} \gets \boldsymbol{\mathcal{B}}_{b},\boldsymbol{\mathrm{\nu}}^{2} \gets \boldsymbol{\mathrm{\nu}}_{b}^{2}$
  
  \textbf{else} $\boldsymbol{\mathcal{B}} \gets \boldsymbol{\mathcal{B}}_{r},\boldsymbol{\mathrm{\nu}}^{2} \gets \boldsymbol{\mathrm{\nu}}_{r}^{2}$ 

  \For{$i \le N$} {
      $\bar{\boldsymbol{\mathrm{X}}}_{i} \gets \Gamma_{\boldsymbol{\mathcal{B}}\to \boldsymbol{\mathrm{I}}_d}(\boldsymbol{\mathrm{\hat{X}}}_i)$; 
      
      $\check{\boldsymbol{\mathrm{X}}}_{i} \gets \mathrm{Exp}_{\boldsymbol{\mathrm{I}}_{d}}\left(\frac{\boldsymbol{\mathrm{s}}}{\sqrt{{\boldsymbol{\mathrm{\nu}}}^{2}+{\boldsymbol{\mathrm{\epsilon}}}}}\left(\mathrm{Log}_{\boldsymbol{\mathrm{I}}_{d}}\left(\bar{\boldsymbol{\mathrm{X}}}_{i}\right)\right)\right)$;
      
      $\Tilde{\boldsymbol{\mathrm{X}}}_i \gets \Gamma_{\boldsymbol{\mathrm{I}}_d\to \boldsymbol{\mathcal{\hat{G}}}}(\check{\boldsymbol{\mathrm{X}}}_{i})$.$\;\;\;\;\;$}
  $\boldsymbol{\mathrm{\Acute{X}}}_i = (\boldsymbol{\mathrm{M}}^{\frac{1}{2}}\Tilde{\boldsymbol{\mathrm{X}}}_i\boldsymbol{\mathrm{M}}^{\frac{1}{2}})^{\frac{1}{\theta}}$ $\;\;\;\;\;$
\end{algorithm}

\subsubsection{RBN under power-deformed GBWM}
The computation of RBN under the power-deformed GBWM is similar to that under BWM, as shown below:
\begin{align}
    \text{Centering:\;} \bar{\boldsymbol{\mathrm{X}}}_{i}&=\mathrm{\widetilde{Exp}}_{\boldsymbol{\mathrm{I}}_{d}}\left(\widetilde{\Gamma}_{\boldsymbol{\mathcal{B}\to \boldsymbol{\mathrm{I}}_{d}}}\left(\mathrm{\widetilde{Log}}_{\boldsymbol{\mathcal{B}}}\left(\boldsymbol{\mathrm{X}}_{i}\right)\right)\right), \label{eq:gbw_center} \\
    \text{Scaling: } \check{\boldsymbol{\mathrm{X}}}_{i}&=\mathrm{\widetilde{Exp}}_{\boldsymbol{\mathrm{I}}_{d}}(\frac{\boldsymbol{\mathrm{s}}}{\sqrt{{\boldsymbol{\mathrm{\nu}}}^{2}+{\boldsymbol{\mathrm{\epsilon}}}}}(\mathrm{\widetilde{Log}}_{\boldsymbol{\mathrm{I}}_{d}}(\bar{\boldsymbol{\mathrm{X}}}_{i}))), \label{eq:gbw_density} \\
    \text{Biasing: } \Tilde{\boldsymbol{\mathrm{X}}}_{i}&=\mathrm{\widetilde{Exp}}_{\boldsymbol{\mathcal{G}}}\left(\widetilde{\Gamma}_{\boldsymbol{\mathrm{I}}_{d}\to \boldsymbol{\mathcal{G}}}\left(\mathrm{\widetilde{Log}}_{\boldsymbol{\mathrm{I}}_{d}}(\check{\boldsymbol{\mathrm{X}}}_{i})\right)\right), \label{eq:gbw_bias}
\end{align}
where $\mathrm{\widetilde{Exp}}, \mathrm{\widetilde{Log}}$ and $\mathrm{\widetilde{\Gamma}}$ represent the Riemannian exponential, logarithm, and parallel transportation defined under the power-deformed GBWM, respectively.
To streamline the process, we find that the GBWBN can be derived directly from the RBN under the standard BWM. As shown in \cite{han2021generalized}, for any  $\boldsymbol{\mathrm{X}}\in\boldsymbol{\mathcal{S}}_{++}^d$ and $\boldsymbol{\mathrm{M}}\in\boldsymbol{{\mathcal{S}}_{++}^d}$, the map $\varphi: (\boldsymbol{{\mathcal{S}}_{++}^d},{g}^{\mathrm{BW}}) \to (\boldsymbol{{\mathcal{S}}_{++}^d},{g}^{\mathrm{GBW}})$, denoted as $\varphi(\boldsymbol{\mathrm{X}})=\boldsymbol{\mathrm{M}}^{\frac{1}{2}}\boldsymbol{\mathrm{X}}\boldsymbol{\mathrm{M}}^{\frac{1}{2}}$, is a Riemannian isometry. 
Therefore, a diffeomorphism, $f = \boldsymbol{\mathrm{\phi}}_{\theta}\circ \varphi$, can be defined, which is a Riemannian isometry from BWM to $\theta^2 g$ with $g$ as $(\theta)$-GBWM. With the BWM-based RBN defined below: 
\begin{equation}
\label{bwbn}
    \text{BWBN}(\boldsymbol{\mathrm{X}}_{i},\boldsymbol{\mathcal{G}},\omega,\boldsymbol{\mathrm{\epsilon}},\boldsymbol{\mathrm{s}}), \forall \boldsymbol{\mathrm{X}}_{i}\in\{\boldsymbol{\mathrm{X}}_{i\cdots N} \in \boldsymbol{{\mathcal{S}}_{++}^d}\},
\end{equation}
the following theorem can be derived.
\begin{theorem}
\label{power-gbwbn}
    Given an SPD manifold $\boldsymbol{{\mathcal{S}}_{++}^d}$, a diffeomorphism $\textit{f}=\boldsymbol{\mathrm{\phi}}_{\theta}\circ \varphi: (\boldsymbol{{\mathcal{S}}_{++}^d}, g^{BW}) \to (\boldsymbol{{\mathcal{S}}_{++}^d}, g^{\theta\text{-GBW}})$, and a batch of SPD matrices $\{{\boldsymbol{\mathrm{X}}_{i}}\}_{i=1}^{N}$, GBWBN can be achieved via the following steps:
    \begin{itemize}
    \item \textbf{Mapping data into } $g^{BW} \mathrm{:}$
    \begin{equation}
    {\boldsymbol{\mathrm{\hat{X}}}}_i = f^{-1}({\boldsymbol{\mathrm{X}}}_i), \quad {\boldsymbol{\mathcal{\hat{G}}}} = f^{-1}(\boldsymbol{\mathcal{G}}).
    \label{map1}
    \end{equation}
    
    \item \textbf{Computing BWBN in } $(\boldsymbol{{\mathcal{S}}_{++}^d }, g^{BW}) \mathrm{:}$
    \begin{equation}
    \Tilde{\boldsymbol{\mathrm{X}}}_i= \text{BWBN}({\boldsymbol{\mathrm{\hat{X}}}}_i,{\boldsymbol{\mathcal{\hat{G}}}},\omega,\boldsymbol{\mathrm{\epsilon}},\boldsymbol{\mathrm{s}}).
    \label{gbwbn_1}
    \end{equation}
    
    \item \textbf{Mapping the normalized data back to } $g^{\theta\text{-GBW}} \mathrm{:}$
    \begin{equation}
    \boldsymbol{\mathrm{\Acute{X}}}_i = f(\Tilde{\boldsymbol{\mathrm{X}}}_i).
    \label{map2}
    \end{equation}
    \end{itemize}
\end{theorem}
\begin{proof}
    The proof is presented in \cref{app:proof-power-gbwbn}.
\end{proof}

\cref{arch: GBWBN} is an overview of the proposed GBWBN algorithm. As shown in \cite{chen2024liebn}, RBN only differs in variance under $\frac{1}{\theta^{2}}g$ and $g$. When calculating the variance in the BWBN, we need to divide the variance by $\theta^{2}$. Now, we give the framework of $(\theta)$-GBWBN in \cref{alg:bn}. Please refer to \cref{train GBWBN} for detailed information about training the GBWBN module. 

\begin{figure}[!t]
 \centering
 \includegraphics[width=0.99\linewidth]{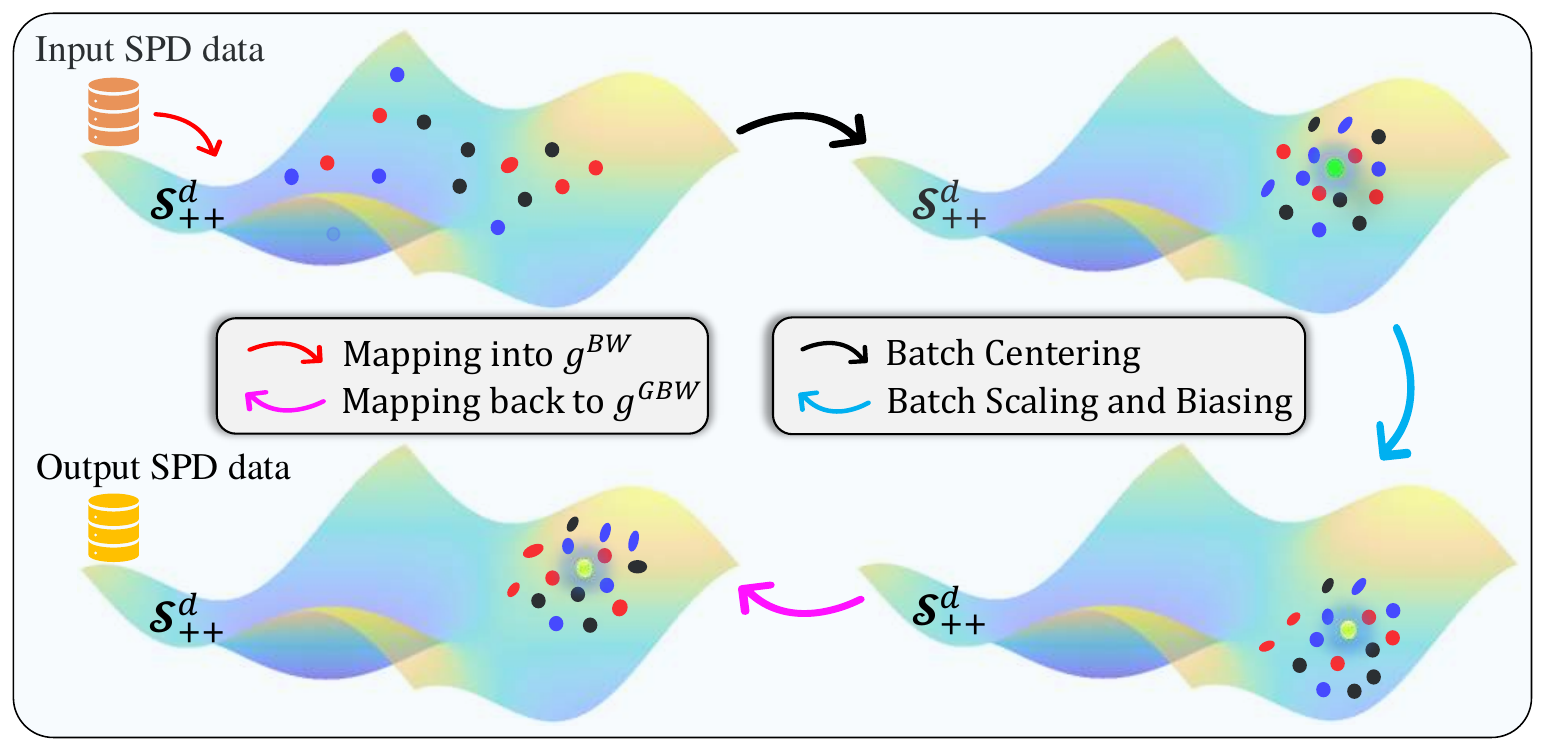}
 \caption{An overview of the computation process for GBWBN.}
 \label{arch: GBWBN}
\end{figure}

%% file: Sections/Experiment.tex
\section{Experiment}
\label{sec:experiments}
To validate the performance of the proposed GBWBN, we integrate it into two different SPD networks and test it on three distinct signal classification tasks: skeleton-based human action recognition using the HDM05 \cite{hdm05} and NTU RGB+D \cite{NtuRGB+D} datasets, and EEG signal classification employing the MAMEM-SSVEP-II dataset \cite{mamem}. 
In the experiments, we denote the network structure as $\{d_0,...,d_r,...,d_{\mathbb{L}}\}$, where $\mathbb{L}$ represents the number of BiMap layers in each SPD backbone network, and $d_{r-1}\times d_r$ signifies the size of the transformation matrix in the $r$-th BiMap layer. If the GBWBN layer does not reach saturation under the standard metric ($\theta=1$), we report the results for the GBWBN under the deformed metric. Detailed descriptions of the data preprocessing and implementation are provided in \cref{datasets detail}.

\subsection{Integration with SPDNet}
\begin{figure*}[!t]
 \centering
 \includegraphics[width=0.95\linewidth]{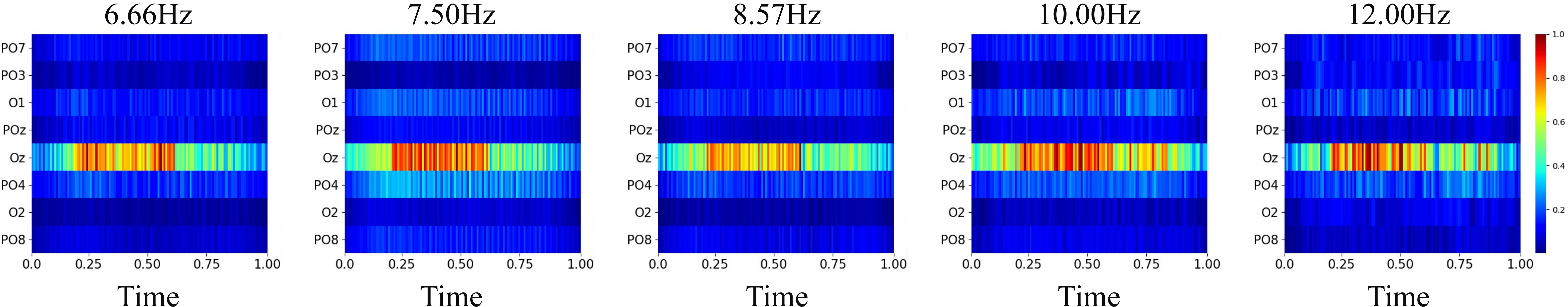}
 \caption{The heatmaps of the absolute gradient responses (computed as in \cite{matt}) for the S11 model across five frequency categories on the MAMEM-SSVEP-II dataset. In each heatmap, the x-axis represents time, while the y-axis signifies the EEG channels. 
 }
 \label{fig-heat}
\end{figure*}

\begin{figure*}[!t]
 \centering
 \includegraphics[width=0.95\linewidth]{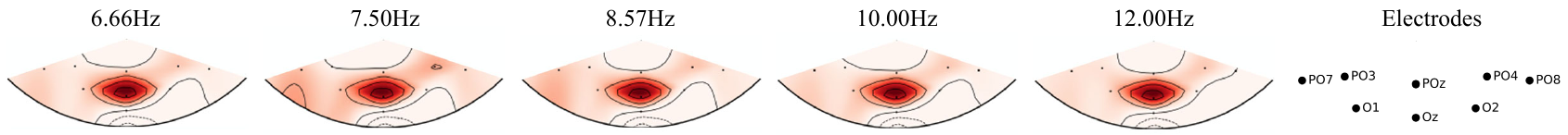}
 \caption{The spatial topomaps of the mean absolute gradient responses (computed as in \cite{matt}) across time for the S11 model on the MAMEM-SSVEP-II dataset. The brain region marked in dark red corresponds to channel Oz, which shows strong gradient activation across the visual cortex for all stimulation frequencies.
 }
 \label{fig-topo}
\end{figure*}

This subsection evaluates GBWBN under the widely adopted SPDNet backbone \cite{spdnet}, choosing the MAMEM-SSVEP-II dataset as an example. 
This dataset \cite{mamem} has 11 subjects and 5 sessions. 
Following the protocol of \citet{matt}, the first three sessions from each subject are used for training, with session 4 for validation and session 5 for testing. Moreover, our SPDNet-GBWBN is evaluated on this dataset under an architecture of \{15, 12\}, a learning rate of $2.5e^{-3}$, and a batch size of 30. 

For the evaluation, in addition to the aforementioned RBN methods, several Euclidean deep learning-based EEG models are also selected for comprehensive comparison, including EEGNet~\cite{matt17}, ShallowConvNet~\cite{matt18}, SCCNet~\cite{matt19}, EEG-TCNet \cite{matt52}, FBCNet \cite{matt53}, TCNet-Fusion \cite{matt51}, and MBEEGSE \cite{matt50}. The average classification results are reported in \cref{mamem_acc}. It can be seen that \textbf{our GBWBN layer improves the vanilla SPDNet by 3.17\%}. Besides, the accuracy of SPDNet-GBWBN is 2.71\%, 4.87\%, 3.39\%, and 5.32\% higher than that of SPDNet-BN (m), SPDNet-BN (m+v), SPDNet-Manifoldnorm(AIM), and SPDNet-LieBN(AIM), respectively. These results show the effectiveness of our GBWBN in parsing complex EEG signals.

\begin{table}[!t] 
    \renewcommand\arraystretch{0.9}
    \centering
    \begin{tabular}{lc}
    \toprule
    {Models} & Acc. (\%) \\
    \midrule
    EEGNet \cite{matt17} & 53.72 $\pm$ 7.23 \\
    ShallowConvNet \cite{matt18} & 56.93 $\pm$ 6.97 \\
    SCCNet \cite{matt19} & 62.11 $\pm$ 7.70 \\
    EEG-TCNet \cite{matt52} & 55.45 $\pm$ 7.66 \\
    FBCNet \cite{matt53} & 53.09 $\pm$ 5.67 \\
    TCNet-Fusion \cite{matt51} & 45.00 $\pm$ 6.45 \\
    MBEEGSE \cite{matt50} & 56.45 $\pm$ 7.27 \\
    \midrule
    SPDNet \cite{spdnet} & 62.30 $\pm$ 3.12 \\
    SPDNet-BN(m) \cite{spdnetbn} &  62.76 $\pm$ 3.01 \\
    SPDNet-BN(m+v) \cite{kobler2022controlling}
     & 60.60 $\pm$ 3.57  \\
    SPDNet-Manifoldnorm \cite{chakraborty2020manifoldnorm} & 62.08 $\pm$ 3.56  \\
    SPDNet-LieBN \cite{chen2024liebn} &  60.15 $\pm$ 3.42 \\
    \midrule
    SPDNet-GBWBN-($\theta=1$) &  \textbf{65.47 $\pm$ 3.33}\\
    \bottomrule
    \end{tabular}
    \caption{Accuracy comparison on the MAMEM-SSVEP-II dataset.}
    \label{mamem_acc}
\end{table}

To reveal the characteristics learned from EEG signals, \cref{fig-heat} visualizes the gradient responses of the S11 model in the channel and temporal domains. The dark red pixels in the heatmaps highlight the robust gradient responses derived from the 8-channel SSVEP EEG data. Besides, the brain topomaps illustrated in \cref{fig-topo} highlight regions in the visual cortex, particularly Oz, with potent gradient activations marked in dark red. Notably, Oz shows the most pronounced gradient responses, especially between 0.25 and 0.60 seconds, underscoring its pivotal role in processing visual information. These findings dovetail with previous studies on the relationship between SSVEP and Oz in EEG research~\cite{han2018highly,mamem-a1}. The position of Oz, right at the heart of the primary visual cortex, is known for its high potential amplitudes and superior signal-to-noise ratio, further supporting these observations. To intuitively demonstrate the proposed power-deformed GBWM-based RBN method is capable of extracting more pivotal geometric features than the AIM-based ones, we also give the visualization results generated by the AIM-based RBN method. The detailed comparison is available at \cref{add}.

\begin{table}[!t]
    \centering
    \resizebox{\linewidth}{!}
    {
    \begin{tabular}{l|c c|c c}
    \toprule
    \multirow{2}{*}{Models} & \multicolumn{2}{c|}{HDM05} & \multicolumn{2}{c}{NTU RGB+D} \\
    \cmidrule(lr){2-3} \cmidrule(lr){4-5}
     & Acc. (\%) & Time & Acc. (\%) & Time \\
    \midrule
    RResNet \cite{RResNet} & 61.09 $\pm$ 0.60 & \textbf{3.62} & 52.54 $\pm$ 0.59 & \textbf{359.25} \\
    RResNet-BN (m) \cite{spdnetbn} & 63.31 $\pm$ 0.61 & 4.49 & 53.86 $\pm$ 1.19 & 478.62 \\
    {RResNet-BN (m+v) \cite{kobler2022controlling}} & 64.51 $\pm$ 1.00 & 6.65 & 53.94 $\pm$ 1.28 & 546.15 \\
    RResNet-Manifoldnorm \cite{chakraborty2020manifoldnorm} & 63.07 $\pm$ 0.80 & 6.24 & 53.50 $\pm$ 0.46 & 531.48 \\
    RResNet-LieBN \cite{chen2024liebn} & 66.43 $\pm$ 0.92 & 5.52 & 54.74 $\pm$ 0.75 & 523.19 \\
    \midrule
    RResNet-GBWBN-($\theta=1$) & 62.01 $\pm$ 1.23 & 8.21 & 59.45$\pm$ 0.38 & 563.47 \\
    RResNet-GBWBN-($\theta=0.5$) & \textbf{69.10 $\pm$ 0.83} & 8.21 & \textbf{59.72 $\pm$ 0.31} & 563.47 \\
    \bottomrule
    \end{tabular}
    }
    \caption{Comparison of different methods on the HDM05 and NTU RGB+D datasets.}
    \label{combined_acc}
    
\end{table}
\subsection{Integration with RResNet}
We further test the availability of the proposed GBWBN on the skeleton-based action recognition task, choosing the HDM05 and NTU RGB+D datasets as two representative benchmarks. For these experiments, the recently developed Riemannian Residual Network (RResNet) \cite{RResNet} is selected as the SPD backbone. The HDM05 dataset \cite{hdm05} contains 2,337 action sequences across 130 distinct classes. To ensure a fair comparison, we conduct experiments using pre-processed covariance features from \cite{spdnetbn}, yielding a reduced dataset comprising 2,086 sequences over 117 classes. The NTU RGB+D \cite{NtuRGB+D} is a large-scale action recognition dataset consisting of 56,880 video clips belonging to 60 different tasks, where we use the flattened 3D joint coordinates as feature vectors, in line with \cite{RResNet}. 
In the experiments, we evaluate the designed RResNet-GBWBN under the network structures of \{93, 30\} and \{75, 30\} for the HDM05 and NTU RGB+D datasets, respectively. We use half of the samples for training and the rest for testing. Besides, the learning rate, number of training epochs, and batch size are set to $2.5e^{-3}$, 200, and 30 on the HDM05 dataset, while those on the NTU RGB+D dataset are configured as $1e^{-2}$, 100, and 256, respectively.

The 10-fold experimental results, along with the average training time (s/epoch) for different methods, are presented in \cref{combined_acc}. We can observe that \textbf{RResNet-GBWBN with proper power deformation respectively improves the vanilla backbone by 8.01\% and 7.18\% } on the HDM05 and NTU RGB+D datasets, further demonstrating the efficacy of the proposed GBWBN module (placed post-BiMap layer) in geometric data analysis. It also can be noted that our GBWBN can generally benefit from power deformation, underscoring the significance of metric parameterization. While the proposed GBWBN module incurs slightly more training time than other RBN methods (involving more SVD operations), it consistently delivers superior performance across both datasets.

\begin{table}[t]
    \centering
    \begin{tabular}{c c c}
    \toprule
    $\lambda$ & \makecell{SPDNet-BN \cite{spdnetbn}} & SPDNet-GBWBN \\
    \midrule
    $1e^{-7}$ & 61.28 $\pm$ 3.89 & \textbf{65.17 $\pm$ 3.87} \\
    $1e^{-6}$ & 61.40 $\pm$ 3.86 & \textbf{65.32 $\pm$ 3.62} \\
    $1e^{-5}$ & 62.76 $\pm$ 3.01 & \textbf{65.47 $\pm$ 3.33} \\
    $1e^{-4}$ & 63.56 $\pm$ 3.44 & \textbf{65.44 $\pm$ 3.27} \\
    $1e^{-3}$ & 62.42 $\pm$ 3.21 & \textbf{65.65 $\pm$ 3.13} \\
    $1e^{-2}$ & 61.89 $\pm$ 3.44 & \textbf{64.56 $\pm$ 3.34} \\
    \bottomrule
    \end{tabular}
    \caption{Accuracy (\%) on the MAMEM-SSVEP-II dataset.}
    \label{tab:lambda_comparison}
\end{table}

\subsection{Ablations}
\textit{\textbf{Statistics of ICSM:}} As discussed in \cref{sec:intro}, given the input covariance matrix $\boldsymbol{\mathrm{X}}$ is not necessarily positive definite, a widely used regularization trick is applied: $\boldsymbol{\mathrm{X}} \leftarrow \boldsymbol{\mathrm{X}} + \lambda \boldsymbol{\mathrm{I}}$. Wherein, $\lambda$ is a small perturbation and $\boldsymbol{\mathrm{I}}$ is the identity matrix. 
This means that the smallest eigenvalues of $\boldsymbol{\mathrm{X}}$ is likely to be $\lambda$, contributing marginally to alleviate the issue of matrix ill-conditioning. Therefore, we compare the designed SPDNet-GBWBN with SPDNet-BN \cite{spdnetbn} on the MAMEM-SSVEP-II dataset across various values of $\mathrm{\lambda}$. As shown in \cref{tab:lambda_comparison}, the accuracy of SPDNet-GBWBN is superior to that of SPDNet-BN in all cases. To further support our claim, we present the condition number ($\kappa={\rm{eig}}_{\rm{max}}/{\rm{eig}}_{\rm{min}}$) and the corresponding proportion on the hidden SPD features without, before, and after the GBWBN layer under different $\lambda$ and training epochs in \cref{combined_lambda_eigenvalue}. \textbf{It is evident that ill-conditioning is prevalent across all the datasets, which explains the superiority of our GBWBN over previous RBN methods, as BWM is advantageous in addressing ill-conditioning} \cite{han2023learning}.

We further investigate the impact of BWM on ICSM to gain a better understanding of the effectiveness of the proposed GBWBN. As shown in \cref{fig:ill-condition}, ill-conditioned SPD matrices are concentrated near the boundary of the SPD manifold, restricting the network’s expressibility. In contrast, with the designed RBN layer, the matrices are shifted toward the spacious interior of the SPD manifold, thereby enhancing the network’s representation power. Besides, as illustrated in \cref{tab:metric}, the AIM-based operators could overly stretch eigenvalues, especially under ill-conditioned scenario, leading to computational instability. In contrast, the BWM is relatively more stable, enabling a more effective ICSM learning.

\begin{table*}[!t]
    \centering
    \resizebox{\linewidth}{!}{
    \begin{tabular}{c|c|c|ccc|ccc|ccc}
    \toprule    
    \multirow{2}{*}{Dataset} & \multirow{2}{*}{$\lambda$} &\multirow{2}{*} {Epochs} & \multicolumn{3}{c|}{$\kappa >10^{3}$} & \multicolumn{3}{c|}{$\kappa >10^{4}$} & \multicolumn{3}{c}{$\kappa >10^{5}$} \\
    \cmidrule(lr){4-12}
      & & &  w/o  & before  & \cellcolor{gray!15}after  & w/o  & before  & \cellcolor{gray!15}after &  w/o  & before  & \cellcolor{gray!15}after \\
    \midrule
    \multirow{12}{*}{MAMEM-SSVEP-II} 
    & \multirow{3}{*}{$1\mathrm{e}^{-7}$}  & 1  & 399 (79.8\%) & 108 (21.6\%) &  \cellcolor{gray!15} 0 (0.0\%)& 195 (39.0\%) & 5 (1.0\%) & \cellcolor{gray!15} 0 (0.0\%) & 74 (14.8\%)  & 0 (0.0\%) & \cellcolor{gray!15} 0 (0.0\%) \\
    &                                      & 100 & 428 (85.6\%) & 17 (3.4\%) & \cellcolor{gray!15} 0 (0.0\%) & 247 (49.4\%) & 0 (0.0\%) & \cellcolor{gray!15} 0 (0.0\%) & 115 (23\%) & 0 (0.0\%) & \cellcolor{gray!15} 0 (0.0\%) \\
    &                                      & 200 (Final) & 458 (91.6\%) & 13 (2.6\%) & \cellcolor{gray!15} 0 (0.0\%) & 274 (100.0\%) & 0 (0.0\%) & \cellcolor{gray!15} 0 (0.0\%) & 113 (54.8\%) & 0 (0.0\%) & \cellcolor{gray!15} 0 (0.0\%) \\
    \cmidrule(lr){2-12}
    & \multirow{3}{*}{$1\mathrm{e}^{-6}$}  & 1 & 375 (75.0\%) & 108 (21.6\%) & \cellcolor{gray!15} 0 (0.0\%) & 191 (38.2\%) & 6 (1.2\%) & \cellcolor{gray!15} 0 (0.0\%) & 64 (12.8\%) & 0 (0.0\%) & \cellcolor{gray!15} 0 (0.0\%) \\
    &                                     & 100 & 384 (76.8\%) & 7 (1.4\%) & \cellcolor{gray!15} 0 (0.0\%) & 201 (40.2\%) & 0 (0.0\%) & \cellcolor{gray!15} 0 (0.0\%) & 67 (13.4\%) & 0 (0.0\%) & \cellcolor{gray!15} 0 (0.0\%) \\
    &                                     & 200 (Final) & 370 (74.0\%) & 40 (8.0\%) & \cellcolor{gray!15} 0 (0.0\%) & 202 (40.4\%) & 0 (0.0\%) & \cellcolor{gray!15} 0 (0.0\%) & 71 (14.2\%) & 0 (0.0\%) & \cellcolor{gray!15} 0 (0.0\%) \\
    \cmidrule(lr){2-12}
    & \multirow{3}{*}{$1\mathrm{e}^{-5}$}  & 1 & 379 (75.8\%) & 119 (23.8\%) & \cellcolor{gray!15} 0 (0.0\%) & 163 (32.6\%) & 8 (1.6\%) & \cellcolor{gray!15} 0 (0.0\%) & 24 (4.8\%) & 0 (0.0\%) & \cellcolor{gray!15} 0 (0.0\%) \\
    &                                      & 100 & 302 (60.4\%) & 15 (3.0\%) & \cellcolor{gray!15} 0 (0.0\%) & 129 (25.8\%) & 0 (0.0\%) & \cellcolor{gray!15} 0 (0.0\%) & 21 (4.2\%) & 0 (0.0\%) & \cellcolor{gray!15} 0 (0.0\%) \\
    &                                      & 200 (Final) & 113 (22.6\%) & 32 (6.4\%) & \cellcolor{gray!15} 0 (0.0\%) & 11 (2.2\%) & 0 (0.0\%) & \cellcolor{gray!15} 0 (0.0\%) & 0 (0.0\%) & 0 (0.0\%) & \cellcolor{gray!15} 0 (0.0\%) \\
    \cmidrule(lr){2-12}
    & \multirow{3}{*}{$1\mathrm{e}^{-4}$}  & 1  & 194 (38.8\%) & 134 (26.8\%) & \cellcolor{gray!15} 0 (0.0\%) & 99 (19.8\%) & 12 (2.4\%) & \cellcolor{gray!15} 0 (0.0\%) & 4 (0.8\%) & 0 (0.0\%) & \cellcolor{gray!15} 0 (0.0\%) \\
    &                                      & 100 & 196 (39.2\%) & 30 (6.0\%) & \cellcolor{gray!15} 0 (0.0\%) & 43 (8.6\%) & 0 (0.0\%) & \cellcolor{gray!15} 0 (0.0\%) & 0 (0.0\%) & 0 (0.0\%) & \cellcolor{gray!15} 0 (0.0\%) \\
    &                                      & 200 (Final) & 205 (41.0\%) & 24 (4.8\%) & \cellcolor{gray!15} 0 (0.0\%) & 56 (11.2\%) & 0 (0.0\%) & \cellcolor{gray!15} 0 (0.0\%) & 1 (0.2\%) & 0 (0.0\%) & \cellcolor{gray!15} 0 (0.0\%) \\
    \midrule
    \multirow{3}{*}{HDM05} 
    & \multirow{3}{*}{$1\mathrm{e}^{-6}$}  & 1 & 2086 (100.0\%) & 2086 (100.0\%) & \cellcolor{gray!15}0 (0.0\%) & 2086 (100.0\%) &  2086 (100.0\%) & \cellcolor{gray!15}0 (0.0\%) & 2086 (100.0\%)  & 2086 (100.0\%) & \cellcolor{gray!15}0 (0.0\%) \\
    &                      & 100 & 2086 (100.0\%) &  2086 (100.0\%) & \cellcolor{gray!15}0 (0.0\%)& 2086 (100.0\%) &  2086 (100.0\%) & \cellcolor{gray!15}0 (0.0\%) & 2086 (100.0\%) &  2086 (100.0\%) & \cellcolor{gray!15}0 (0.0\%) \\
    &                      & 200 (Final) & 2086 (100.0\%) & 2086 (100.0\%) & \cellcolor{gray!15}0 (0.0\%) & 2086 (100.0\%) &  2086 (100.0\%) & \cellcolor{gray!15}0 (0.0\%) & 2086 (100.0\%) &  2086 (100.0\%) & \cellcolor{gray!15}0 (0.0\%) \\
    \midrule
    
    \multirow{3}{*}{NTU RGB+D} 
    & \multirow{3}{*}{$1\mathrm{e}^{-6}$}  & 1 & 56,880 (100.0\%) & 56,880 (100.0\%) & \cellcolor{gray!15}0 (0.0\%) & 56,880 (100.0\%) & 56,880 (100.0\%) & \cellcolor{gray!15}0 (0.0\%) & 56,880 (100.0\%) & 56,880 (100.0\%) & \cellcolor{gray!15}0 (0.0\%) \\
    &                      &  50 & 56,880 (100.0\%) & 56,880 (100.0\%) & \cellcolor{gray!15}0 (0.0\%) & 56,880 (100.0\%) & 56,880 (100.0\%) & \cellcolor{gray!15}0 (0.0\%) & 56,880 (100.0\%) & 56,880 (100.0\%) & \cellcolor{gray!15}0 (0.0\%) \\
    &                      & 100 (Final) & 56,880 (100.0\%) & 56,880 (100.0\%) & \cellcolor{gray!15}0 (0.0\%) & 56,880 (100.0\%) & 56,880 (100.0\%) & \cellcolor{gray!15}0 (0.0\%) & 56,880 (100.0\%) & 56,880 (100.0\%) & \cellcolor{gray!15}0 (0.0\%) \\ 
    \bottomrule
    \end{tabular}
    }
    \caption{Statistics (quantity and proportion) on the condition number ($\kappa={\rm{eig}}_{\rm{max}}/{\rm{eig}}_{\rm{min}}$) of the SPD features without ('w/o'), before, and after the GBWBN layer on different datasets across various values of $\lambda$ and training epochs.}
    \label{combined_lambda_eigenvalue}
\end{table*}

\begin{figure}
    \centering
    \includegraphics[width=0.90\linewidth]{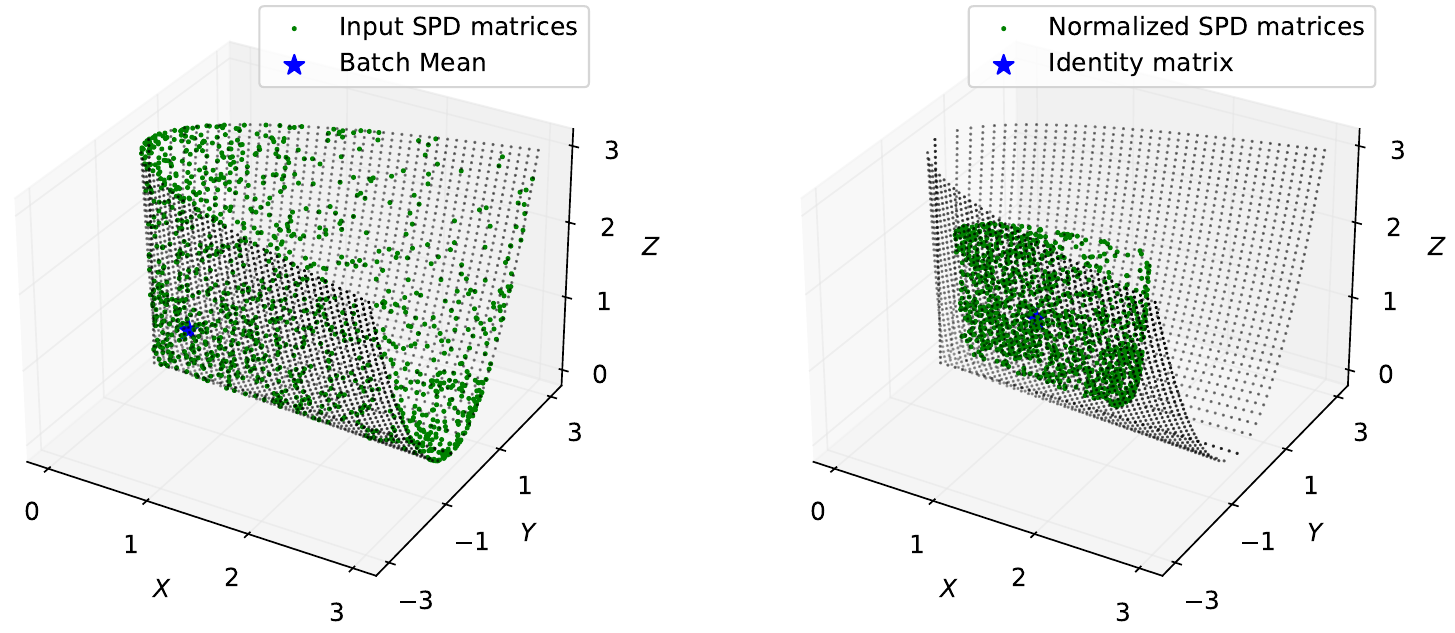}
    \caption{Illustration of ill-conditioned $2 \times 2$ SPD matrices (left) and the output (right) of the GBWBN layer. The black dots are semi-positive matrices ($\kappa=\infty$), denoting the SPD boundary, while the interior of the cone is the SPD manifold.}
    \label{fig:ill-condition}
\end{figure}

\begin{figure}
    \centering
    \includegraphics[width=0.98\linewidth,trim={0 0.5cm 0 0}]{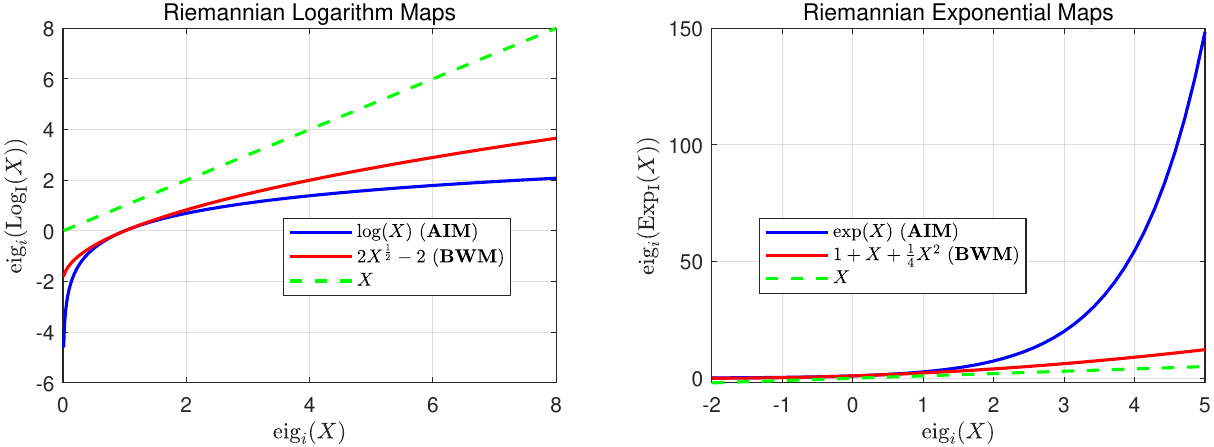}
    \caption{Numerical comparison of the Log and Exp maps under AIM and BWM, where we select the identity matrix as the base point for simplicity. Here, $x$-axis denotes an eigenvalue and $y$-axis is the corresponding output eigenvalue.} 
    \label{tab:metric}
\end{figure}

\textit{\textbf{Ablation on the SPD dimension:}} We conduct experiments on the NTU RGB+D dataset to investigate the impact of SPD dimension on the performance of the GBWBN. As shown in \cref{ntu_bimap_acc}, removing the BiMap layer (which increases the SPD dimension) results in a slight decrease in the learning ability of both RResNet-GBWBN and ResNet-BN. Nevertheless, the accuracy of RResNet-GBWBN-'w/o'BiMap is still higher than that of RResNet-BN-'w/o'BiMap.

\begin{table}[!t]
    \centering
    \begin{tabular}{cc}
    \bottomrule
    \multicolumn{1}{c}{Methods} & Acc. (\%)  \\
    \midrule 
    RResNet-GBWBN & \textbf{59.45 $\pm$ 0.38} \\
    RResNet-GBWBN-'w/o'BiMap  & 57.32 $\pm$ 0.43 \\
    RResNet-BN \cite{spdnetbn}  & 53.86 $\pm$ 1.19 \\
    RResNet-BN-'w/o'BiMap &  52.24 $\pm$ 1.33 \\
    \bottomrule
    \end{tabular}
    \caption{Accuracy comparison on the NTU RGB+D dataset.} 
    \label{ntu_bimap_acc}
\end{table}

\begin{table}[!t]
    \centering
    \resizebox{\linewidth}{!}{
    \begin{tabular}{c|cc|c}
    \toprule
    Datasets  & RResNet & RResNet-BWBN  & RResNet-GBWBN  \\
    \midrule
    HDM05 & 61.09 $\pm$ 0.60  & 61.94 $\pm$ 0.62 & \textbf{62.01 $\pm$ 1.23}  \\
    NTU RGB+D & 52.54 $\pm$ 0.59 & 57.40 $\pm$ 0.52 & \textbf{59.45 $\pm$ 0.38} \\
    \bottomrule
    \end{tabular}}
    \caption{Accuracy (\%) comparison under different RResNets.}
    \label{tab:bwm_comparison}
\end{table}

\begin{table}[!t]
    \centering
    \resizebox{\linewidth}{!}{
    \begin{tabular}{c|cccc}
    \toprule
    $\mathrm{\theta}$ & 0.25 & 0.5 & 0.75 & 1 \\
    \midrule
    HDM05 & 65.28 $\pm$ 1.03 & \textbf{69.10 $\pm$ 0.8}3 & 62.89 $\pm$ 0.74 & 62.01 $\pm$ 1.23 \\
    NTU RGB+D & 58.23 $\pm$ 0.28  & 59.72 $\pm$ 0.31 & \textbf{59.82 $\pm$ 0.27}  & 59.45 $\pm$ 0.38 \\
    \bottomrule
    \end{tabular}}
    \caption{Accuracy (\%) comparison under various values of $\theta$.}
    \label{ntu_theta_acc}
\end{table}

\textit{\textbf{Ablation on the GBWBN:}} As shown in \cref{sec:Propose}, GBWBN represents a principled generalization of the BWM-based RBN (simplified as BWBN). Therefore, we make experiments on the HDM05 and NTU RGB+D datasets to explore the effectiveness of the generalized BWM in RBN design. 

The results listed in \cref{tab:bwm_comparison} show that both GBWBN and BWBN can intensify the classification ability of the vanilla SPD backbone. Notably, the performance of BWBN is inferior to that of GBWBN. This further supports the notion that the learnable extension applied to BWBN allows our RBN algorithm to flexibly capture batch-specific metric geometry, thereby improving its representational capacity.

\textit{\textbf{Ablation on the power deformation:}} In this paper, we propose the power-deformed GBWM, which interpolates between GBWM ($\theta=1$) and LEM-like metric $(\theta \to 0)$. Thereby, we study the impact of such deformation on the performance of RBN under different $\theta$, selecting the HDM05 and NTU RGB+D datasets as two examples. The results given in \cref{ntu_theta_acc} show that the nonlinear activation of the metric geometry can enhance the model accuracy.

%% file: Sections/Conclusion.tex
We introduce a novel RBN algorithm grounded in the GBW geometry in this article. Compared to other metrics, the GBWM offers distinct advantages in computing ICSM, which we believe provides an edge for the proposed GBWBN over existing methods. In addition, the power deformation is employed to further refine and enhance the latent GBWM geometry. Extensive experiments and ablation studies validate the superiority of our approach over previous RBN methods.

%% file: Sections/Ackonwledge.tex
This work was supported in part by the National Natural Science Foundation of China (62306127, 62020106012, 62332008), the Natural Science Foundation of Jiangsu Province (BK20231040), the Fundamental Research Funds for the Central Universities (JUSRP124015), and the National Key R\&D Program of China (2023YFF1105102, 2023YFF1105105).

%% file: Sections/Appendix.tex
\clearpage
\appendix
\setcounter{page}{1}
\maketitlesupplementary

\section{Implementation details}
\label{datasets detail}
\textbf{HDM05} dataset \cite{hdm05} comprises 2,337 action sequences, spanning 130 distinct classes. To ensure a fair comparison, we conduct experiments using pre-processed covariance features provided by \cite{spdnetbn}, which results in a reduced dataset comprising 2,086 sequences spread over 117 classes. The size of the input SPD matrices is $93\times 93$, reflecting the 3D coordinates of 31 body joints provided in each skeleton frame. We use half of the samples for training and the rest for testing. 

\textbf{NTU RGB+D} \cite{NtuRGB+D} is another action recognition dataset which comprises 25 body joints. This large dataset concludes with 56,880 videos and 60 tasks. Following the approach of \cite{RResNet}, we use the flattened versions of 3D joint coordinates as our feature vectors. Therefore, the size of the input SPD matrices is $75 \times 75$. We also use half of the samples for training and the rest for testing. 

\textbf{MAMEM-SSVEP-II} dataset \cite{mamem} consists of time-synchronized EEG recordings from 11 participants, collected via an EGI 300 geodesic EEG system (256 channels at a 250 Hz sampling rate). For the SSVEP-based task, participants were asked to concentrate on any one of five distinct visual stimuli flickering at different frequencies (6.66, 7.50, 8.57, 10.00, and 12.00 Hz). This was done for five seconds across a series of sessions. Each session covered five cue-triggered trials for every stimulation frequency. Each trial, in turn, was broken down into four 1-second segments from the cue's onset (1s-5s). This protocol generated a total of 100 trials per session. To ensure an equitable comparison, we adhere to the data preparation and performance evaluation procedures described in \cite{matt17}.

Following  \cite{RResNet}, we evaluate our GBWBN under the network structure of \{93, 30\} and \{75, 30\} for the HDM05 and NTU RGB+D datasets respectively, with the GBWBN module embedded after the BiMap layer. The learning rate, training epoch, and batch size are set to $2.5e^{-3}$, 200, 30 on the HDM05 dataset, respectively. For the NTU RGB+D dataset, we set the learning rate to $1e^{-2}$, train for 100 epochs, and use a batch size of 256. 

For the MAMEM-SSVEP-II dataset, the criterion used in~\cite{matt19} is applied for data processing. Firstly, two convolutional layers (ConvLs) are adopted at the front of SPDNet models to extract more effective spatiotemporal representations of the original EEG signals. Then, grouping in the channel dimension results in an output feature matrix of size $15\times 126$, from which a $15\times 15$ SPD matrix can be derived. Moreover, the designed SPDNet-GBWBN is evaluated using an architecture of \{15, 12\} along with the training epoch
of 200, the learning rate of $2.5e^{-3}$ and the batch size of 30 on this dataset.

\section{Training the GBWBN module}
\label{train GBWBN}
Noting that the bias $\boldsymbol{\mathcal{G}}$ belongs to the SPD manifold, the traditional Stochastic Gradient Descent (SGD) algorithm fails to respect its Riemannian geometry during optimization. From a geometric viewpoint, choosing the Riemannian Stochastic Gradient Descent (RSGD) algorithm \cite{rsgd,spdmlr} proves to be an effective and rational strategy for optimizing manifold parameters. Let ${L}$ be the loss function of the GBWBN layer, we can have the following:
\begin{equation}
\label{eq:grad}
\begin{split}
    \boldsymbol{\mathcal{G}}_{t+1} ={\mathrm{Exp}}_{\boldsymbol{\mathcal{G}}_t}\left(-\mu\Pi_{\boldsymbol{\mathcal{G}}_t}\left(\nabla_{\boldsymbol{\mathcal{G}}}{L}|_{ \boldsymbol{\mathcal{G}}_t}\right)\right),
\end{split}
\end{equation}  
where $\mu$ denotes the learning rate, $\Pi$ signifies the projection operator used to transform the Euclidean gradient ($\nabla_{\boldsymbol{\mathcal{G}}}{L}|_{ \boldsymbol{\mathcal{G}}_t})$, and ${\mathrm{Exp}}$ is the exponential mapping shown in \cref{opera}.

In previous literature, distinct formulas for the gradient of eigenvalue functions on SPD matrices have been independently established~\cite{ionescu2015training}. 
Considering an eigenvalue function: $\boldsymbol{\mathrm{Y}}=f(\boldsymbol{\mathrm{X}})=\boldsymbol{\mathrm{U}}f(\boldsymbol{\mathrm{\Sigma}})\boldsymbol{\mathrm{U}}^{\mathrm{T}}$, where $\boldsymbol{\mathrm{X}}=\boldsymbol{\mathrm{U}}\boldsymbol{\mathrm{\Sigma}}\boldsymbol{\mathrm{U}}^{\mathrm{T}}$,
the gradient of $L$ w.r.t $\boldsymbol{\mathrm{X}}$ can be computed as follows:
\begin{equation}
\label{eq:svd}
\begin{split}
    \frac{\partial L}{\partial \boldsymbol{\mathrm{X}}}= \boldsymbol{\mathrm{U}}\left({f}' (\boldsymbol{\mathrm{X}})\odot \left(\boldsymbol{\mathrm{U}}^{T}(\frac{\partial L}{\partial \boldsymbol{\mathrm{Y}}})\boldsymbol{\mathrm{U}}\right)\right)\boldsymbol{\mathrm{U}}^{T},
\end{split}
\end{equation}
with 
\begin{equation}
\label{eq:dif}
\begin{split}
    f'(\boldsymbol{\mathrm{X}})_{i,j}=\left\{\begin{array}{ll}\frac{f\left(\sigma_{i}\right)-f\left(\sigma_{j}\right)}{\sigma_{i}-\sigma_{j}}, & \text {if }\sigma_{i}\neq\sigma_{j}. \\f^{\prime}\left(\sigma_{i}\right),&\text{otherwise}.
\end{array}\right.
\end{split}
\end{equation}
\cref{eq:svd} is called the Daleck\u ii-Kre\u in formula and $f'(\lambda_{i},\lambda_{j})$ is the first divided difference of $f$ at $(\lambda_{i},\lambda_{j})$. We acknowledge the work presented in ~\cite{spdnetbn} which illustrates the equivalence between the Daleck\u ii-Kre\u in formula and the formula introduced in ~\cite{ionescu2015training}. Given the superior numerical stability of the Daleck\u ii-Kre\u formula, we utilize this formula throughout the backward pass.

Given the following Lyapunov equation:
\begin{equation}
\label{Lyeq}
\begin{split}
    {\boldsymbol{\mathrm{P}}\boldsymbol{\mathrm{X}}+\boldsymbol{\mathrm{X}}\boldsymbol{\mathrm{P}}}={\boldsymbol{\mathrm{S}}}, \text{ with } \boldsymbol{\mathrm{P}}=\mathcal{L}_{\boldsymbol{\mathrm{X}}}(\boldsymbol{\mathrm{S}}),
\end{split}
\end{equation}
according to ~\cite{chen2024rmlr}, the gradients w.r.t. the Lyapunov operator in the backpropagation process can be computed as follows:
\begin{align}
\label{Lydif2}
    &\frac{\partial \mathcal{L}}{\partial \boldsymbol{\mathrm{S}}}=\mathcal{L}_{\boldsymbol{\mathrm{X}}}(\frac{\partial \mathcal{L}}{\partial \boldsymbol{\mathrm{P}}}),\\&\frac{\partial \mathcal{L}}{\partial \boldsymbol{\mathrm{X}}}=-\boldsymbol{\mathrm{P}}\mathcal{L}_{\boldsymbol{\mathrm{X}}}(\frac{\partial \mathcal{L}}{\partial \boldsymbol{\mathrm{P}}})-\mathcal{L}_{\boldsymbol{\mathrm{X}}}(\frac{\partial \mathcal{L}}{\partial \boldsymbol{\mathrm{P}}})\boldsymbol{\mathrm{P}}.
\end{align}

\section{Learning the bias}
\label{app: learn bias}
The bias operation $\boldsymbol{\mathrm{X}}_{i}^{bias}=\rieexp_{\boldsymbol{\mathcal{G}}}\left(\Gamma_{\boldsymbol{\mathrm{I}}_{d}\to \boldsymbol{\mathcal{G}}}\left(\rielog_{\boldsymbol{\mathrm{I}}_{d}}\left(\boldsymbol{\mathrm{X}}_{i}\right)\right)\right)$ can be implemented by three auxiliary layers: 1) $\boldsymbol{\mathrm{Y}}_{i}=f^{(1)}(\boldsymbol{\mathrm{X}}_{i})=\rielog_{\boldsymbol{\mathrm{I}}_{d}}(\boldsymbol{\mathrm{X}}_{i})$; 2) $\boldsymbol{\mathrm{Z}}_{i}=f^{(2)}(\boldsymbol{\mathrm{Y}}_{i},\boldsymbol{\mathcal{G}})=\Gamma_{\boldsymbol{\mathrm{I}}_{d}\to \boldsymbol{\mathcal{G}}}(\boldsymbol{\mathrm{Y}}_{i})$; 3) $\boldsymbol{\mathrm{X}}_{i}^{bias}=f^{(3)}(\boldsymbol{\mathrm{Z}}_{i},\bscal{G}) =\rieexp_{\boldsymbol{\mathcal{G}}}(\boldsymbol{\mathrm{Z}}_{i})$. Therefore, the bias operation can be expressed as:
\begin{equation}
\label{biasstr}
\begin{split}
    \boldsymbol{\mathrm{X}}_{i}\overset{f^{(1)}}{\longrightarrow}\boldsymbol{\mathrm{Y}}_{i}\overset{f^{(2)}}{\longrightarrow}\boldsymbol{\mathrm{Z}}_{i}\overset{f^{(3)}}{\longrightarrow}\boldsymbol{\mathrm{X}}_{i}^{bias}.
\end{split}
\end{equation}
\begin{figure}[!t]
 \centering
 \includegraphics[scale=0.3]{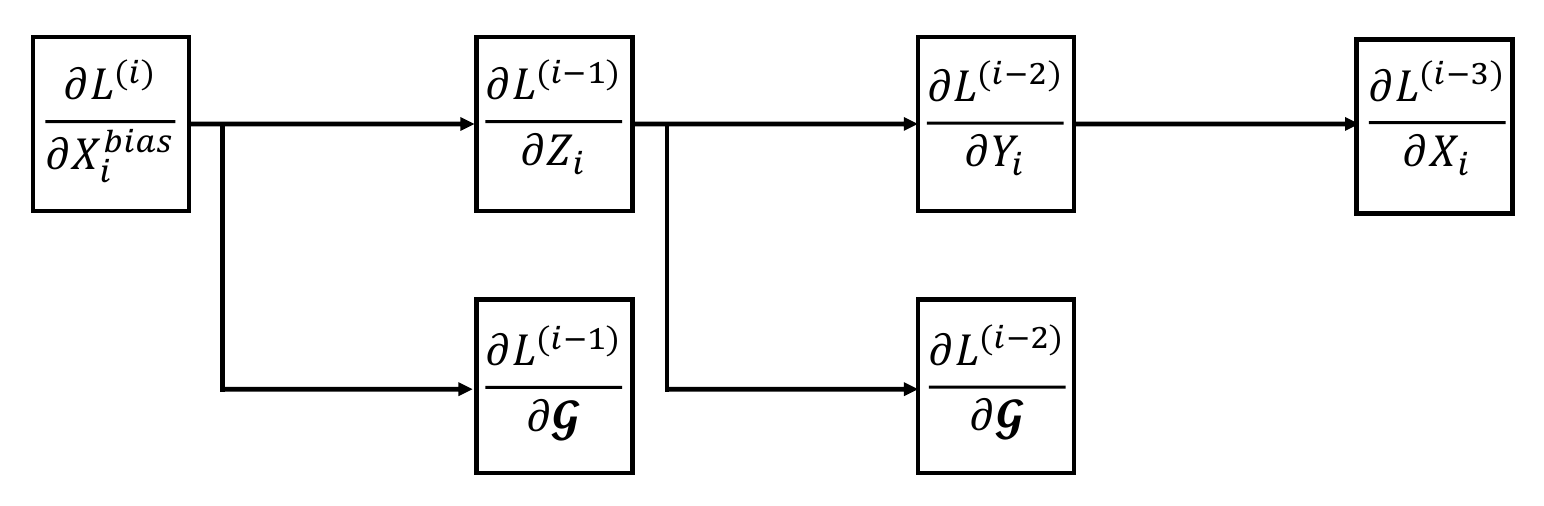}
 \caption{The backward pass of bias operation}
 \label{fig-bp}
\end{figure} 

Let $L:\mathbb{R}^{d}\rightarrow \mathbb{R}$ be the loss function of the designed GBWBN. The backpropagation of the bias operation
is illustrated in \cref{fig-bp}, where $L^{(i)}=L\circ f^{(K)}\circ\dots\circ f^{(i)}$ represents the loss of the $i$-th layer. It is crucial to note that the computation of the partial derivatives of $f^{(3)}$ w.r.t. $\bsym{Z}_{i}$ and $\bscal{G}$ involves the backpropagation of Lyapunov operator. The partial derivative of $f^{(2)}$ w.r.t. $\bscal{G}$ is computed automatically by PyTorch. Based on the above computations, we can deduce the Euclidean gradient ($\nabla_{\boldsymbol{\mathcal{G}}}{L}|_{ \boldsymbol{\mathcal{G}}_t}$) in the proposed BWBN module. Then, \cref{eq:grad} can be used to update the bias. Since the partial derivative of $f^{(1)}$ w.r.t. $\bsym{X}_{i}$ involves the computation of square root $(\cdot )^{\frac{1}{2}}$, following SPDNetBN~\cite{spdnetbn}, we use \cref{eq:dif} to solve it.

\section{Proofs of the propositions and theories in the main paper}
\label{proof:main}

\subsection{Proof of Prop. 3.1}
\label{app:proof-control_var}
Using Riemannian distance in \cref{opera}, we can deduce that
{\small
\begin{equation}
\label{eq:control_var_deduce}
\begin{split}
    {d_{\mathrm{BW}}}(\boldsymbol{\mathrm{X}},\boldsymbol{\mathrm{I}}_{d})&=\left(\mathrm{tr}({\boldsymbol{\mathrm{X}}})+\mathrm{tr}(\boldsymbol{\mathrm{I}}_{d})-2\mathrm{tr}({\boldsymbol{\mathrm{X}}^{\frac{1}{2}}\boldsymbol{\mathrm{I}}_{d}\boldsymbol{\mathrm{X}}^{\frac{1}{2}})}^{\frac{1}{2}}\right)^{\frac{1}{2}}\\&=\left(\mathrm{tr}({\boldsymbol{\mathrm{X}}})+\mathrm{tr}(\boldsymbol{\mathrm{I}}_{d})-2\mathrm{tr}(\boldsymbol{\mathrm{X}})^{\frac{1}{2}}\right)^{\frac{1}{2}}\\&=\left(\mathrm{tr}\left(\left({\boldsymbol{\mathrm{X}}}^{\frac{1}{2}}-\boldsymbol{\mathrm{I}}_{d}\right)^{2}\right)\right)^{\frac{1}{2}}.
\end{split}
\end{equation}
}

Similarly, using the exponential map and logarithmic map in \cref{opera}, we can deduce that
\begin{equation}
\label{eq: exp_bwm}
\begin{split}
    \mathrm{Exp}_{\boldsymbol{\mathrm{I}}_{d}}\boldsymbol{\mathrm{S}}=\boldsymbol{\mathrm{I}}_{d}+\boldsymbol{\mathrm{S}}+\mathcal{L}_{\boldsymbol{\mathrm{I}}_{d}}(\boldsymbol{\mathrm{S}})\boldsymbol{\mathrm{I}}_{d}\mathcal{L}_{\boldsymbol{\mathrm{I}}_{d}}(\boldsymbol{\mathrm{S}}),
\end{split}
\end{equation}
\begin{equation}
\label{eq: log_bwm}
\begin{split}
    \mathrm{Log}_{\boldsymbol{\mathrm{I}}_{d}}\boldsymbol{\mathrm{X}}=2\boldsymbol{\mathrm{X}}^{\frac{1}{2}}-2\boldsymbol{\mathrm{I}}_{d}.
\end{split}
\end{equation}

Since $\boldsymbol{\mathrm{I}}_{d}\mathcal{L}_{\boldsymbol{\mathrm{I}}_{d}}(\boldsymbol{\mathrm{S}})+\mathcal{L}_{\boldsymbol{\mathrm{I}}_{d}}(\boldsymbol{\mathrm{S}})\boldsymbol{\mathrm{I}}_{d}=\boldsymbol{\mathrm{S}}$, we can obtain: $\mathcal{L}_{\boldsymbol{\mathrm{I}}_{d}}(\boldsymbol{\mathrm{S}})=\frac{1}{2}\boldsymbol{\mathrm{S}}$. Then, \cref{eq: exp_bwm} can be reformulated as:
\begin{equation}
\label{eq: exp_bwm_re}
\begin{split}
    \mathrm{Exp}_{\boldsymbol{\mathrm{I}}_{d}}\boldsymbol{\mathrm{S}}=\boldsymbol{\mathrm{I}}_{d}+\boldsymbol{\mathrm{S}}+\frac{1}{4}\boldsymbol{\mathrm{S}}^{2}=(\boldsymbol{\mathrm{I}}_{d}+\frac{1}{2}\boldsymbol{\mathrm{S}})^{2}.
\end{split}
\end{equation}

Combining \cref{eq: log_bwm} and \cref{eq: exp_bwm_re}, we have
\begin{equation}
\label{eq:scale}
\begin{split}
    \boldsymbol{\mathrm{\psi}}_{\boldsymbol{\mathrm{s}}}(\boldsymbol{\mathrm{X}})&=\mathrm{Exp}_{\boldsymbol{\mathrm{I}}_{d}}\left(\boldsymbol{\mathrm{s}}\mathrm{Log}_{\boldsymbol{\mathrm{I}}_{d}}\left(\boldsymbol{\mathrm{X}}\right)\right)\\&=\mathrm{Exp}_{\boldsymbol{\mathrm{I}}_{d}}\left(\boldsymbol{\mathrm{s}}\left(2\boldsymbol{\mathrm{X}}^{\frac{1}{2}}-2\boldsymbol{\mathrm{I}}_{d}\right)\right)\\&=\left(\boldsymbol{\mathrm{s}}\left(\boldsymbol{\mathrm{X}}^{\frac{1}{2}}-\boldsymbol{\mathrm{I}}_{d}\right)+\boldsymbol{\mathrm{I}}_{d}\right)^{2}
\end{split}
\end{equation}

Therefore,
\begin{equation}
\label{eq:scale_proof}
\begin{split}               
   {d_{\mathrm{BW}}} \left(\boldsymbol{\mathrm{\psi}}_{\boldsymbol{\mathrm{s}}}\left(\boldsymbol{\mathrm{X}}\right),\boldsymbol{\mathrm{I}}_{d}\right)
    &=\left(\mathrm{tr}\left(\left(\boldsymbol{\mathrm{s}}\left(\boldsymbol{\mathrm{X}}^{\frac{1}{2}}-\boldsymbol{\mathrm{I}}_{d}\right)+\boldsymbol{\mathrm{I}}_{d}-\boldsymbol{\mathrm{I}}_{d}\right)^{2}\right)\right)^{\frac{1}{2}}\\
    &=\left(\mathrm{tr}\left(\boldsymbol{\mathrm{s}}^{2}\left({\boldsymbol{\mathrm{X}}}^{\frac{1}{2}}-\boldsymbol{\mathrm{I}}_{d}\right)^{2}\right)\right)^{\frac{1}{2}}\\
    &=\boldsymbol{\mathrm{s}}\left(\mathrm{tr}\left(\left({\boldsymbol{\mathrm{X}}}^{\frac{1}{2}}-\boldsymbol{\mathrm{I}}_{d}\right)^{2}\right)\right)^{\frac{1}{2}}\\
    &=\boldsymbol{\mathrm{s}} {d_{\mathrm{BW}}}\left(\boldsymbol{\mathrm{X}},\boldsymbol{\mathrm{I}}_{d}\right)
\end{split}
\end{equation}

\subsection{Proof of Prop. 3.2}
\label{app:proof-deformed_0}

As shown in \cite{chen2024liebn}, when $\theta \to 0$, for all $\boldsymbol{\mathrm{X}} \in {{\mathcal{S}}_{++}^d}$ and all $\boldsymbol{\mathrm{S}} \in {T}_{\boldsymbol{\mathrm{X}}}{\boldsymbol{{\mathcal{S}}_{++}^d}}$, we have
{\small
\begin{equation}
    \label{eq:deformed_i}
     g^{(\theta)-\mathrm{GBW}}_{\boldsymbol{\mathrm{X}}} \left(\boldsymbol{\mathrm{S}},\boldsymbol{\mathrm{S}}\right) \longrightarrow g^{\mathrm{GBW}}_{\boldsymbol{\mathrm{I}}_{d}}\left(\log_{*,\boldsymbol{\mathrm{X}}}\left(\boldsymbol{\mathrm{S}}\right),\log_{*,\boldsymbol{\mathrm{X}}}\left(\boldsymbol{\mathrm{S}}\right)\right).
\end{equation}
}
Using \cref{eq:metric_gbwm}, we can deduce that 
\begin{equation}
\label{eq:metric_gbwm_i}
\begin{split}
    {g}_{\boldsymbol{\mathrm{I}}_{d}}^{\mathrm{GBW}}\left(\boldsymbol{\mathrm{S}},\boldsymbol{\mathrm{S}}\right)&= \frac{1}{2}{\mathrm{tr}}\left(\mathcal{L}_{\boldsymbol{\mathrm{I}}_{d},\boldsymbol{\mathrm{M}}}\left(\boldsymbol{\mathrm{S}}\right)\boldsymbol{\mathrm{S}}\right)\\&=\frac{1}{2}{\mathrm{tr}}\left(\mathcal{L}_{\boldsymbol{\mathrm{M}}}\left(\boldsymbol{\mathrm{S}}\right)\boldsymbol{\mathrm{S}}\right)\\&= \frac{1}{2}\left \langle \mathcal{L}_{\boldsymbol{\mathrm{M}}}\left(\boldsymbol{\mathrm{S}}\right),\boldsymbol{\mathrm{S}}\right \rangle 
\end{split}
\end{equation}
Combining \cref{eq:deformed_i} and \cref{eq:metric_gbwm_i}, we have 
{\small
\begin{equation}
\label{final_proof_pro2}
\begin{split}
    g^{(\theta)-\mathrm{GBW}}_{\boldsymbol{\mathrm{X}}} \left(\boldsymbol{\mathrm{S}},\boldsymbol{\mathrm{S}}\right) &\overset{\theta \to 0}{\longrightarrow} g^{\mathrm{GBW}}_{\boldsymbol{\mathrm{I}}_{d}}\left(\log_{*,\boldsymbol{\mathrm{X}}}\left(\boldsymbol{\mathrm{S}}\right),\log_{*,\boldsymbol{\mathrm{X}}}\left(\boldsymbol{\mathrm{S}}\right)\right)\\
    &=
    \frac{1}{2} \left \langle \log_{*,\boldsymbol{\mathrm{X}}}\left(\mathcal{L}_{\boldsymbol{\mathrm{M}}}\left(\boldsymbol{\mathrm{S}}\right)\right),\log_{*,\boldsymbol{\mathrm{X}}}\left(\boldsymbol{\mathrm{S}}\right) \right \rangle 
\end{split}
\end{equation}
}

\subsection{Proof of Prop. 3.3}
\label{app:proof-local_deform}

For $(\theta)$-GBW, We have
\begin{equation}
\label{local_deform_1}
\begin{split}
    g^{(\theta)-\mathrm{GBW}}_{\boldsymbol{\mathrm{X}}} &\left(\boldsymbol{\mathrm{S}}_{1},\boldsymbol{\mathrm{S}}_{2}\right)\\&=\frac{1}{\theta^2} g_{\boldsymbol{\mathrm{X}}^\theta}^{\mathrm{GBW}}\left(\Tilde{\boldsymbol{\mathrm{S}}}_{1},\Tilde{\boldsymbol{\mathrm{S}}}_{2}\right)
    \\&= \frac{1}{4} \cdot \frac{1}{\theta^2} \mathrm{vec}(\Tilde{\boldsymbol{\mathrm{S}}}_{1})^\top (\Tilde{\boldsymbol{\mathrm{X}}} \otimes \Tilde{\boldsymbol{\mathrm{X}}} )^{-1} \mathrm{vec}(\Tilde{\boldsymbol{\mathrm{S}}}_{2}) \\&\stackrel{(1)}{=} \frac{1}{4} \cdot \frac{1}{\theta^2} g ^{\mathrm{AI}}_{\Tilde{\boldsymbol{\mathrm{X}}}}(\boldsymbol{\mathrm{S}}_{1},\boldsymbol{\mathrm{S}}_{2})\\&=\frac{1}{4}g^{(\theta)-\mathrm{AI}}_{\boldsymbol{\mathrm{X}}}\left(\boldsymbol{\mathrm{S}}_{1},\boldsymbol{\mathrm{S}}_{2}\right),
\end{split}
\end{equation}

where $\Tilde{\boldsymbol{\mathrm{S}}}_{1}=(\phi_\theta)_{*,\boldsymbol{\mathrm{X}}}(\boldsymbol{\mathrm{S}}_{1}), \Tilde{\boldsymbol{\mathrm{S}}}_{2}=(\phi_\theta)_{*,\boldsymbol{\mathrm{X}}}(\boldsymbol{\mathrm{S}}_{2}), \Tilde{\boldsymbol{\mathrm{X}}}=\boldsymbol{\mathrm{X}}^{\theta}$, and (1) follows from \cref{eq:metric_aim}.

\subsection{Proof of Thm. 3.4}
\label{app:proof-power-gbwbn} 

Our proof is inspired by \cite[Thm. 5.3]{chen2024liebn}, which characterize the calculation of LieBN \cite{{chen2024liebn}} under the Riemannian isometry between Lie groups.
However, our RBN does not involve Lie group structures but relies on Riemannian operators. 
In the following, we will use the properties of Riemannian operators under the Riemannian isometry.

We denote $\mathrm{Exp}, \mathrm{Log}, \mathrm{\Gamma}, d(\cdot, \cdot)$ and $\mathrm{WFM}$ are Riemannian exponentiation, logarithm, parallel transportation, geodesic distance, and weighted Fréchet mean on $(\boldsymbol{{\mathcal{S}}_{++}^d},{g}^{\mathrm{BW}})$, while $\widetilde{\mathrm{Exp}}, \widetilde{\mathrm{Log}},\widetilde{\mathrm{\Gamma}},\widetilde{\mathrm{d}}(\cdot, \cdot)$ and $\widetilde{\mathrm{WFM}}$ are the counterparts on $(\boldsymbol{{\mathcal{S}}_{++}^d}, g^{\theta\text{-GBW}})$. Since $\textit{f}: (\boldsymbol{{\mathcal{S}}_{++}^d}, g^{BW}) \to (\boldsymbol{{\mathcal{S}}_{++}^d}, g^{\theta\text{-GBW}})$ is a Riemannian isometry, for $\boldsymbol{\mathrm{X}}_{1}, \boldsymbol{\mathrm{X}}_{2} \in \boldsymbol{{\mathcal{S}}_{++}^d} $, $\boldsymbol{\mathrm{S}}_{1}, \boldsymbol{\mathrm{S}}_{2} \in{T}_{\boldsymbol{\mathrm{X}}_{1}}\boldsymbol{{\mathcal{S}}_{++}^d}$, and $\{{\boldsymbol{\mathrm{X}}_{i}}\}_{i=1}^{N}\in \boldsymbol{{\mathcal{S}}_{++}^d}$ with weights $\{\omega_{i}\}_{i=1}^{N}$  satisfying ${\omega }_{i}\ge {0}$ and $\sum_{i\le N}{\omega}_{i}= 1$ .we have the following:
\begin{align}
    & \widetilde{\mathrm{Exp}}_{\bsym{X}_{1}}\bsym{S}_{1}=f\left ( \mathrm{Exp}_{f^{-1}(\bsym{X}_{1})}{f^{-1}}_{*,\bsym{X}_{1}}(\bsym{S}_{1}) \right ),\label{eq:equal_exp} \\
    & \widetilde{\mathrm{Log}}_{\bsym{X}_{1}}{\bsym{X}_{2}}=(f^{-1}_{*,\bsym{X}_{1}})^{-1}\left(\mathrm{Log}_{f^{-1}(\bsym{X}_{1})}f^{-1}(\bsym{X}_{2})\right),\label{eq:equal_log} 
\end{align}
\begin{equation}
\begin{split}
    & \widetilde{\Gamma}_{\bsym{X}_{1}\to\bsym{X}_{2}}{\bsym{S}_{1}}\\&=\left(f^{-1}_{*,\bsym{X}_{2}}\right)^{-1}\left( \Gamma_{f^{-1}(\bsym{X}_{1})\to f^{-1}(\bsym{X}_{2})}{f^{-1}}_{*,\bsym{X}_{1}}(\bsym{S}_{1})\right),\label{eq:equal_pt}  
\end{split}
\end{equation}
\begin{align}
    & \widetilde{d}(\bsym{X}_{1},\bsym{X}_{2})=d\left(f^{-1}(\bsym{X}_{1}),f^{-1}(\bsym{X}_{2})\right),\label{eq:equal_dist} \\
    & \widetilde{\mathrm{WFM}}(\{{\boldsymbol{\mathrm{X}}_{i}}\},\{ \omega_{i}\})=f(\mathrm{WFM}(\{f^{-1}({\boldsymbol{\mathrm{X}}_{i}})\},\{ \omega_{i}\})). \label{eq:equal_wfm}
\end{align}

Note that:
\begin{equation}
\begin{split}
    &\widetilde{\mathrm{Exp}}_{\bsym{X}_2}\left(\widetilde{\Gamma}_{\bsym{X}_{1}\to \bsym{X}_{2}}{\widetilde{\mathrm{Log}}_{\bsym{X}_{1}}({\bsym{S}_{1}})}\right)\\
    &\stackrel{(1)}{=}    \widetilde{\mathrm{Exp}}_{\bsym{X}_2}\left(\widetilde{\Gamma}_{\bsym{X}_{1}\to \bsym{X}_{2}}(f^{-1}_{*,\bsym{X}_{1}})^{-1}\left(\mathrm{Log}_{f^{-1}(\bsym{X}_{1})}f^{-1}(\bsym{S}_{1})\right)\right)\\
    &\stackrel{(2)}{=}\widetilde{\mathrm{Exp}}_{\bsym{X}_2}(f^{-1}_{*,\bsym{X}_{2}})^{-1}\left( \Gamma_{f^{-1}(\bsym{X}_{1})\to f^{-1}(\bsym{X}_{2})}\left(\bscal{A}\right)\right)\\
    &\stackrel{(3)}{=}f\left ( \mathrm{Exp}_{f^{-1}(\bsym{X}_{1})} \left( \Gamma_{f^{-1}(\bsym{X}_{1})\to f^{-1}(\bsym{X}_{2})}\left(\bscal{A}\right)\right) \right ),
\end{split}
\end{equation}
where $\bscal{A} = \mathrm{Log}_{f^{-1}(\bsym{X}_{1})}f^{-1}(\bsym{S}_{1})$.

The derivation comes from the following.

(1)  follows from \cref{eq:equal_log}.

(2)  follows from \cref{eq:equal_pt}.

(3)  follows from \cref{eq:equal_exp}.

Then, we denote Eq. (11) and Eq. (13)  on $(\boldsymbol{{\mathcal{S}}_{++}^d}, g^{BW})$ of the main paper as $\xi(\cdot|\bscal{B},\nu^{2},\bscal{G},\bsym{s})$, while $\widetilde{\xi}(\cdot|\bscal{B},\nu^{2},\bscal{G}, \bsym{s})$ is the counterpart on $(\boldsymbol{{\mathcal{S}}_{++}^d}, g^{\theta\text{-GBW}})$. We can deduce that: 
\begin{equation}
\label{eq:equal_map}
\begin{split}
     &\widetilde{\xi}(\bsym{X}_{i}|\bscal{B},\nu^{2},\bscal{G},\bsym{s})\\&=f\left(\xi\left(f^{-1}(\bsym{X}_{i})|f^{-1}(\bscal{B}),\nu^{2},f^{-1}(\bscal{G}),\bsym{s}\right)\right).
\end{split}
\end{equation}

Since \textit{f} is a Riemannian isometry, we can directly deduce that the Fréchet variance and Fréchet mean are both the same. Therefore, we can obtain that:

\begin{equation}
\begin{split}
\label{eq:equal_bn}
   &\selectfont(\theta)\text{-GBWBN}(\boldsymbol{\mathrm{X}}_{i},\boldsymbol{\mathcal{G}},\omega,\boldsymbol{\mathrm{\epsilon}},\boldsymbol{\mathrm{s}})\\&=f\left(\text{BWBN}\left(\boldsymbol{f^{-1}(\mathrm{X}}_{i}),f^{-1}(\boldsymbol{\mathcal{G}}),\omega,\boldsymbol{\mathrm{\epsilon}},\boldsymbol{\mathrm{s}}\right)\right).
\end{split}
\end{equation}

\section{EEG model interpretation}
\label{add}

\begin{figure*}[!t]
 \centering
 \includegraphics[width=0.9\linewidth]{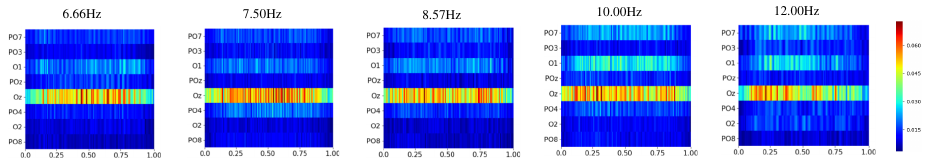}
 \caption{ The heatmaps of five frequency classes of the MAMEM-SSVEP-II dataset demonstrated by SPDNet-BN. 
 }
 \label{fig-heat1}
\end{figure*}

\begin{figure*}[!t]
 \centering
 \includegraphics[width=0.9\linewidth]{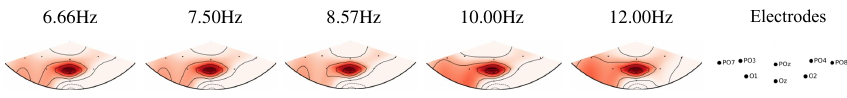}
 \caption{ The brain topology maps obtained on the MAMEM-SSVEP-II dataset demonstrated by SPDNet-BN.
 }
 \label{fig-topo1}
\end{figure*}

\cref{fig-heat1} and \cref{fig-topo1} are the visualization results generated by the SPDNet-BN \cite{spdnetbn}. Compared with \cref{fig-heat} and \cref{fig-topo}, the gradient responses of our method and SPDNet-BN are both concentrated in the OZ channel, clearly demonstrating the effectiveness of the Riemannian method. However, our method shows a concentration between 0.25 and 0.60 seconds, dovetailed with previous studies on the relationship between SSVEP and Oz in EEG research ~\cite{han2018highly,mamem-a1}, whereas SPDNet-BN exhibits a less focused response. This observation further supports that the proposed power-deformed GBWM-based RBN can extract more essential geometric features than AIM-based methods.

\cref{topo} illustrates the spatial distribution across distinct epochs corresponding to each visual stimulus in the MAMEM-SSVEP-II dataset. Each epoch displays similar spatial topo maps with a consistently strong gradient response located at the Oz spanning all epochs. Moreover, despite varying frequencies of visual stimulation, the gradient response location across the scalp remained steady throughout the given period. 

In conclusion, SPDNet-GBWBN can identify subtle discrepancies hidden within similar spatial distributions represented by the topographic map of each epoch among the five frequencies used to decode SSVEP-EEG signals. The visualization results validate the proficiency and capability of the proposed SPDNet-GBWBN in identifying and capturing the elusive non-stationarity observed in dynamic brain activity.

\begin{figure*}[!t]
 \centering
 \includegraphics[width=0.5\textwidth]{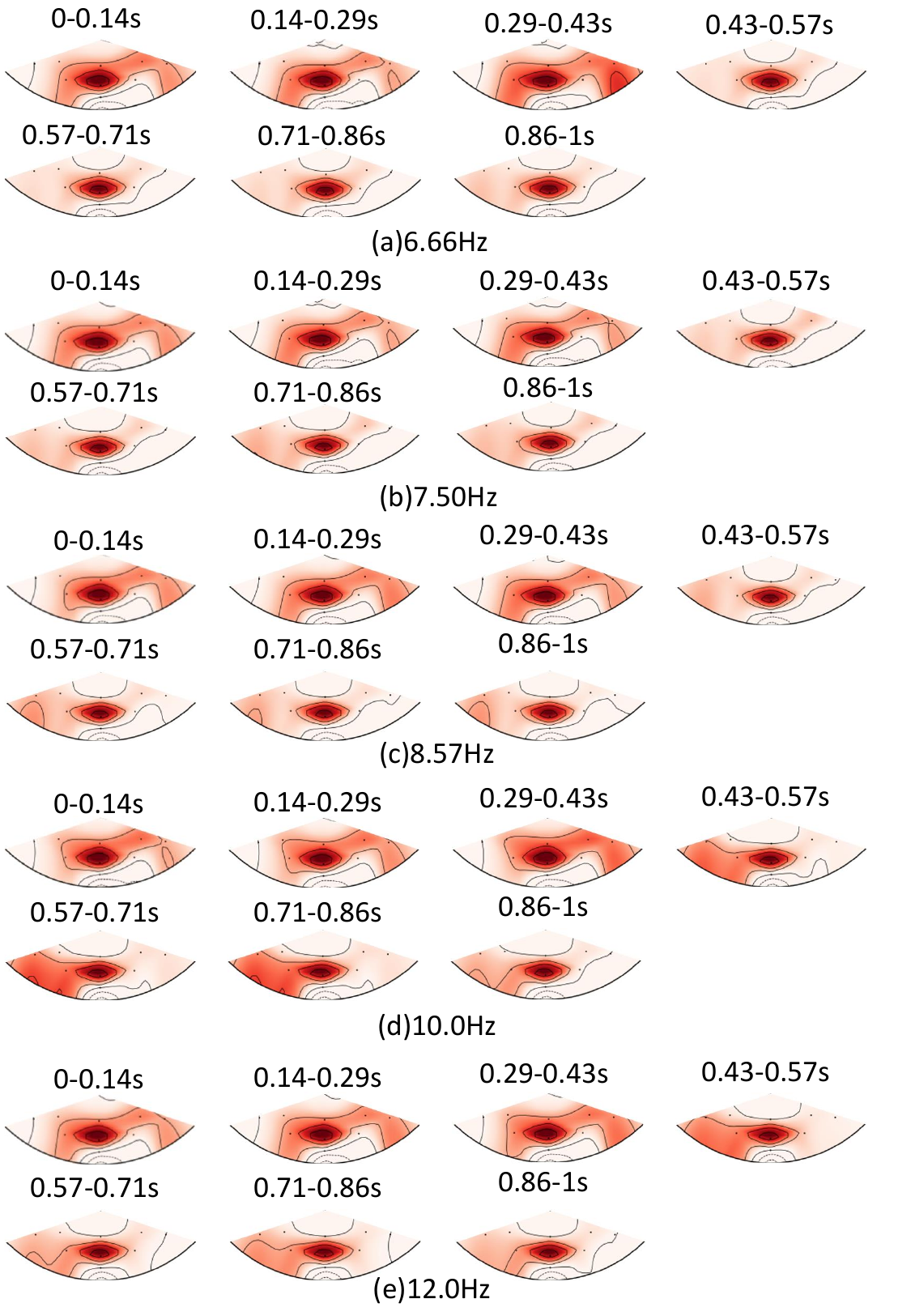}
 \caption{The spatial topomaps at different epochs and frequencies of visual stimulation for the S11 model on the MAMEM-SSVEP-II dataset (dark red signifies strong gradient response).}
 \label{topo}
\end{figure*} 